\documentclass[11pt]{article}

\usepackage[T1]{fontenc}
\usepackage[utf8]{inputenc}
\usepackage[a4paper,margin=1in]{geometry}
\usepackage{amsmath,amssymb,amsfonts,amsthm}
\usepackage{booktabs}
\usepackage{graphicx}
\usepackage{multirow}
\usepackage{multicol}
\usepackage{tablefootnote}
\usepackage{float}
\usepackage{url}
\usepackage{makecell}
\usepackage{xspace}
\usepackage{tikz}
\usepackage{subcaption}
\usepackage[numbers,sort&compress]{natbib}
\usepackage{hyperref}
\usepackage[capitalize,noabbrev]{cleveref}
\usepackage{microtype}
\usepackage{setspace}
\usepackage{authblk}

\graphicspath{{figures/}}

\newcommand{\revised}[1]{#1}
\newcommand{\modelname}{\textsc{Q-BioLat}\xspace}
\newtheorem{proposition}{Proposition}[section]
\newtheorem{example}{Example}[section]

\title{\modelname: Binary Latent Protein Fitness Landscapes for QUBO-Based Optimization}
\author[1]{Truong-Son Hy\thanks{Corresponding author: \href{mailto:thy@uab.edu}{thy@uab.edu}}}
\affil[1]{Department of Computer Science, University of Alabama at Birmingham, Birmingham, Alabama, USA}
\date{}

\begin{document}
\maketitle

\begin{abstract}
\noindent\textbf{Background:}
Protein fitness optimization is a discrete search problem, and the representation used for prediction also determines the neighborhood graph traversed by an optimizer. We introduce \modelname, a framework that maps pretrained protein-language-model embeddings to compact binary codes and fits a quadratic unconstrained binary optimization (QUBO) surrogate with unary and pairwise latent interactions. Our central contribution is an optimization-aware view of representation: binary encodings that are similar in pointwise predictive accuracy can induce different Hamming neighborhoods, local optima, and search trajectories. We formalize when a recoding is only a Hamming-isometric reparameterization and give a constructive example showing that exact pointwise agreement does not imply optimization equivalence.

\medskip
\noindent\textbf{Results:}
We study experimentally measured GFP and AAV fitness landscapes from ProteinGym. The internal QUBO surrogate is evaluated against labels withheld from QUBO fitting. A conservative retrieval analysis maps optimized codes to measured variants and reports their experimental fitness, while neural decoding of potentially unmeasured sequences is evaluated separately with an experiment-trained sequence surrogate and is interpreted only as model-based candidate prioritization. Across the reported comparisons, PCA followed by per-coordinate median thresholding yields a more balanced and decodable binary space than the post-hoc-zero-threshold AE/VAE baselines, despite the latter's low continuous reconstruction error. In the measured-library retrieval analysis, simulated annealing, genetic algorithms, and greedy hill climbing frequently return high-percentile variants; decoded candidates are reported separately using surrogate-predicted scores.

\medskip
\noindent\textbf{Conclusions:}
\modelname provides a compact framework for studying how representation, pairwise energy structure, and discrete search jointly affect protein-fitness optimization; prospective validation of newly decoded sequences remains future work.
\end{abstract}


\section{Introduction} \label{sec:Introduction}

\begin{figure*}
    \centering
    \includegraphics[width=0.9\linewidth]{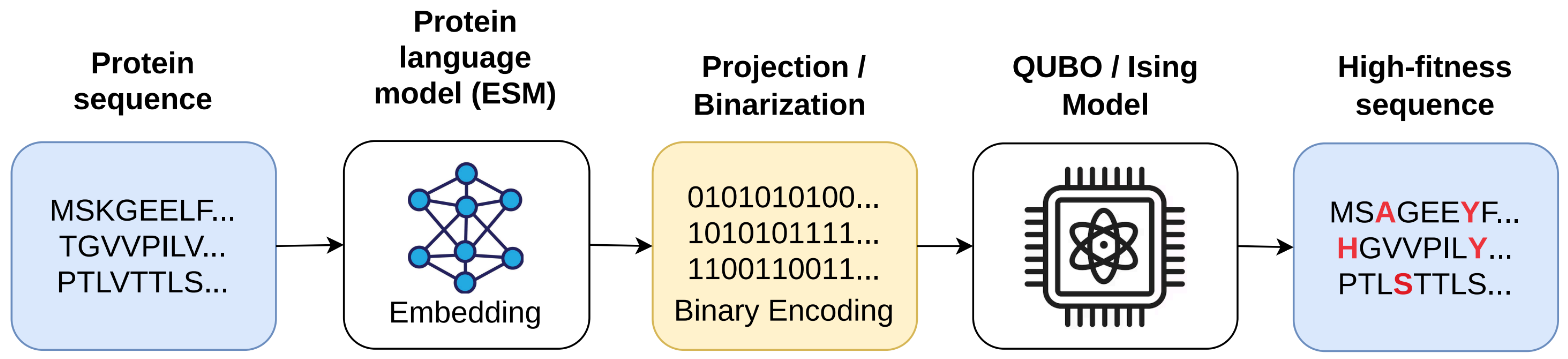}
    \caption{Overview of the \modelname framework. Protein sequences are first encoded using a pretrained protein language model (ESM) to obtain continuous embeddings. These embeddings are transformed into binary latent representations through projection and binarization, enabling protein fitness to be modeled as a quadratic unconstrained binary optimization (QUBO) problem. The resulting latent fitness landscape can be explored using combinatorial optimization methods and is directly compatible with quantum annealing hardware. Optimized latent codes are mapped back to candidate protein sequences.}
    \label{fig:Main_Figure}
\end{figure*}


Understanding and optimizing protein fitness landscapes is a central challenge in computational biology, with applications spanning enzyme engineering, drug discovery, and synthetic biology. Protein fitness landscapes are typically high-dimensional, rugged, and shaped by complex interactions between residues, making efficient exploration of sequence space difficult \cite{romero2009exploring, CHEN2023706, ProteinGym}. Recent advances in protein language models (PLMs), such as ESM-2 \cite{ESM-2} and ESM-3 \cite{ESM-3}, have enabled powerful representations of protein sequences by learning from large-scale unlabeled data. These representations have significantly improved performance on downstream tasks, including mutation effect prediction \cite{10.1093/bioinformatics/btae621, liu2025plm} and protein design \cite{watson2023novo, LatentDE, DiscreteDE}.

Despite these advances, most existing approaches rely on continuous neural predictors and gradient-based \cite{LatentDE} or sampling-based optimization strategies \cite{DiscreteDE}, which are not naturally suited for discrete combinatorial search. In particular, protein sequences are inherently discrete objects, and many biologically relevant optimization problems are fundamentally combinatorial in nature. Classical approaches such as evolutionary algorithms \cite{DiscreteDE, OZAWA2024101758} and Bayesian optimization \cite{atkinson2025protein, 10.1093/bib/bbac570} have been widely applied to protein engineering, but they often struggle with scalability or require large numbers of evaluations. More importantly, existing methods primarily emphasize predictive accuracy, while the role of the underlying representation in shaping the \emph{optimization landscape} remains underexplored.

In this work, we introduce \modelname, a framework that maps protein sequences into binary latent spaces where fitness landscapes can be explicitly modeled and optimized (see Figure \ref{fig:Main_Figure}). Starting from pretrained protein language model embeddings, we construct compact binary representations that enable the formulation of protein fitness as a quadratic unconstrained binary optimization (QUBO) \cite{kochenberger2014unconstrained, IsingFormulation, kim2025quantum} problem. This formulation transforms protein fitness modeling into a structured energy landscape over binary variables, allowing us to leverage a wide range of combinatorial optimization techniques, including simulated annealing and genetic algorithms, for efficient exploration of protein fitness landscapes.

Beyond its algorithmic formulation, \modelname provides a representation-centric perspective on protein fitness modeling through the lens of induced optimization landscapes. We show that representations with similar predictive performance can exhibit markedly different optimization behavior, as determined by the geometry of the latent space and the structure of the corresponding QUBO interactions. In particular, we observe a clear distinction between \emph{prediction quality} and \emph{optimization effectiveness}: learned autoencoder-based representations collapse after binarization, producing degenerate latent codes that fail to support combinatorial search, whereas simple structured representations such as PCA yield high-entropy and optimization-friendly latent spaces.

\revised{
The central distinction in this work is between pointwise prediction and optimization geometry. Given an encoding \(\phi:\mathcal S\rightarrow\{0,1\}^{m}\), a one-bit optimizer does not see only the predicted values \(\hat f(\phi(s))\); it also sees the representation-induced neighborhood graph
$ E_{\phi} = \left\{ (s,t):d_H\!\left(\phi(s),\phi(t)\right)=1 \right\}. $
Pointwise regression metrics evaluate predictions on observed variants but do not determine the edges of this graph, the values assigned to unobserved codes, or collisions created by a many-to-one encoding. Consequently, a binary representation is an innocuous reparameterization only when it preserves the relevant Hamming geometry. Coordinate permutations and bit complements do so; general projections, thresholding maps, and learned encoders need not. We formalize this distinction in Proposition \ref{prop:search_equivalance} and illustrate it in Figure \ref{fig:encoding_search_geometry} with two exact QUBO fits of the same four fitness values that possess different one-bit local maxima.
}

Empirically, we demonstrate that PCA-based binary latent representations consistently outperform both random projections and learned AE/VAE representations in decoding quality and end-to-end protein design performance across multiple datasets and data regimes. Furthermore, classical combinatorial optimization methods, including simulated annealing, genetic algorithms, and greedy hill climbing, are highly effective when applied to these structured latent landscapes, reliably identifying high-fitness regions even in data-scarce settings. We also observe that optimization performance depends on latent dimensionality, revealing a trade-off between expressivity, generalization, and searchability.

Overall, \modelname bridges protein language modeling, discrete latent representation learning, and combinatorial optimization by reframing protein fitness modeling as the construction of an optimization-aware energy landscape. By expressing protein fitness landscapes as QUBO problems, our framework provides a natural interface between modern machine learning models and discrete optimization methods, while also offering a pathway toward integration with quantum and quantum-inspired optimization paradigms.

\section{Related Work} \label{sec:related_works}

\textbf{Protein language models and sequence representations.}
Recent advances in protein language models (PLMs) have significantly improved the representation of biological sequences by leveraging large-scale unlabeled protein data. Models such as ESM-2 \cite{ESM-2} and its successor ESM-3 \cite{ESM-3} learn contextual embeddings that capture structural and functional properties of proteins directly from sequence data. More recently, multimodal protein language models have been developed to integrate sequence information with additional modalities such as structure, function, and evolutionary context, further enhancing representation quality and downstream performance \cite{10.1093/biomethods/bpae043, Khang_Ngo_2024}. These models have demonstrated strong performance across a wide range of tasks, including structure prediction, mutation effect prediction, and protein design \cite{doi:10.1073/pnas.2016239118, ESM-2}. In this work, we build upon PLM embeddings as a foundation, but transform them into binary latent representations to enable discrete optimization. \\

\noindent
\textbf{Protein fitness prediction and engineering.}
Predicting protein fitness landscapes is a fundamental problem in computational biology, with applications in enzyme optimization, therapeutic design, and synthetic biology. Recent approaches combine machine learning models with experimental data, particularly from deep mutational scanning (DMS) experiments, to predict the effects of mutations \cite{yang2019machine, biswas2021low}. Large-scale benchmarks such as ProteinGym have been introduced to systematically evaluate fitness prediction models across diverse proteins and mutation regimes \cite{ProteinGym}. While these methods achieve strong predictive performance, they typically rely on continuous representations and do not directly address the combinatorial nature of sequence optimization. Furthermore, these approaches primarily focus on prediction rather than explicitly enabling structured optimization over discrete sequence spaces. \\

\noindent
\textbf{Optimization methods for protein design.}
A variety of optimization strategies have been applied to protein engineering, including evolutionary algorithms \cite{LatentDE, DiscreteDE, OZAWA2024101758, 10.1093/bib/bbac570}, Bayesian optimization \cite{10.1093/bib/bbac570, atkinson2025protein}, and reinforcement learning \cite{wang2025reinforcement, sun2025accelerating}. Evolutionary methods such as genetic algorithms are widely used due to their ability to explore large discrete search spaces, while Bayesian optimization offers sample-efficient search under expensive evaluation settings. Classical optimization techniques such as simulated annealing have also been applied to combinatorial problems with rugged energy landscapes \cite{kirkpatrick1983optimization, 10.7551/mitpress/1090.001.0001}. However, many of these approaches operate directly in sequence space or continuous embedding spaces, rather than explicitly modeling protein fitness as a structured combinatorial optimization problem. In contrast, our approach explicitly models protein fitness as a structured combinatorial optimization problem in a learned binary latent space. \\

\noindent
\textbf{QUBO formulations and quantum annealing.}
Quadratic unconstrained binary optimization (QUBO) \cite{kochenberger2014unconstrained, IsingFormulation, kim2025quantum} provides a unified framework for expressing a wide range of combinatorial problems in terms of binary variables and pairwise interactions. QUBO formulations are closely related to the Ising model and form the basis of many optimization techniques, including quantum annealing \cite{IsingFormulation}. Quantum annealing hardware, such as that developed by D-Wave Systems, has demonstrated the ability to solve certain classes of optimization problems by exploiting quantum effects \cite{Johnson2011QuantumAW}. While applications of quantum optimization in computational biology remain limited, recent work has begun to explore its potential for tasks such as molecular design and protein folding. In this work, we formulate protein fitness landscapes as QUBO problems in a binary latent space, enabling direct compatibility with both classical combinatorial solvers and emerging quantum annealing hardware, and providing a unified perspective on protein design as discrete energy landscape optimization. \\

\noindent
\textbf{Latent space optimization and discrete representations.}
Optimizing in learned latent spaces has become a common strategy in machine learning, particularly for generative modeling and design tasks \cite{hottung2021learning, 10.1145/3520304.3533993}. Continuous latent spaces have been widely used for molecule and protein generation \cite{doi:10.1021/acsomega.2c03264, MGVAE, LatentDE}, but they often require gradient-based optimization or sampling methods that do not naturally handle discrete constraints. More recently, discrete and binary latent representations have been explored for enabling combinatorial search and efficient optimization \cite{jiang2023efficient, abe2026effectiveness}. Our work builds on this idea by constructing binary latent representations derived from protein language models and explicitly modeling the resulting fitness landscape as a QUBO problem, bridging representation learning and combinatorial optimization. \\

\noindent
\textbf{Energy landscapes and combinatorial optimization.}
Many combinatorial optimization problems can be interpreted through the lens of energy landscapes, where the objective function defines a surface over discrete configurations and optimization corresponds to navigating this landscape toward low-energy (or high-fitness) regions \cite{10.1093/acprof:oso/9780198570837.001.0001, Wales_2004}. In particular, QUBO formulations are closely related to Ising models and have been widely studied in statistical physics as models of interacting binary systems \cite{IsingFormulation}. The structure of the interaction matrix determines key properties of the landscape, such as smoothness, ruggedness, and the distribution of local optima, which in turn strongly influence optimization difficulty \cite{https://doi.org/10.1002/cpa.21422, pmlr-v38-choromanska15}. Such energy landscape perspectives have also been explored in machine learning, particularly in the analysis of non-convex optimization and neural network loss surfaces \cite{pmlr-v38-choromanska15, 43404}. While these ideas are well established in physics and optimization, their role in modern representation learning pipelines for biological sequence design remains underexplored. In this work, we adopt an energy landscape view of protein fitness modeling in binary latent spaces and study how representation-induced structure affects combinatorial search. \\

\noindent
\textbf{Representation learning for optimization.}
Most existing representation learning approaches in protein modeling and related domains focus on improving predictive performance on downstream tasks, such as regression or classification \cite{10.1093/bioinformatics/btae621, liu2025plm, doi:10.1073/pnas.2016239118}. However, in design and optimization settings, the learned representation implicitly defines a search space and induces an optimization landscape. Recent work in latent space optimization has explored how learned embeddings can facilitate search in applications such as molecule design, protein engineering, and combinatorial optimization \cite{doi:10.1021/acscentsci.7b00572, pmlr-v80-jin18a, C9SC04026A, abeer2024multi}. These approaches typically leverage continuous latent spaces learned by variational autoencoders or related generative models, enabling gradient-based or sampling-based search. 
In protein design, similar ideas have been explored through latent generative models and language-model-guided optimization \cite{LatentDE, DiscreteDE, watson2023novo}, where representations are used to guide exploration of sequence space. However, these approaches are primarily evaluated through predictive metrics, likelihood, or sample quality, rather than explicit analysis of the induced optimization landscape. More broadly, recent work in machine learning has begun to study how representations influence optimization and search behavior \cite{10.1109/TPAMI.2013.50, pmlr-v80-bojanowski18a}, but this perspective remains underdeveloped in the context of biological sequence design.
In contrast, our work emphasizes that representation quality should also be assessed based on the properties of the induced optimization landscape, and demonstrates that different representations with similar predictive accuracy can lead to substantially different optimization outcomes.
\section{Methods}

\subsection{Overview}

We propose \modelname, a framework for modeling and optimizing protein fitness landscapes in a binary latent space. Given a protein sequence, \modelname first computes a pretrained protein language model embedding, then transforms this continuous representation into a compact binary latent code. In this binary space, the fitness landscape is approximated by a quadratic unconstrained binary optimization (QUBO) surrogate, which models both unary and pairwise effects among latent variables. The learned surrogate is then optimized using combinatorial search methods such as simulated annealing and genetic algorithms. Because the final objective is expressed in QUBO form, the framework is naturally compatible with both classical combinatorial solvers and quantum annealing hardware.

Formally, the overall pipeline is
\[
s \;\rightarrow\; e(s) \in \mathbb{R}^{d}
\;\rightarrow\; z(s) \in \mathbb{R}^{m}
\;\rightarrow\; x(s) \in \{0,1\}^{m}
\;\rightarrow\; \hat{f}(x),
\]
where $s$ is a protein sequence, $e(s)$ is the protein language model embedding, $z(s)$ is a reduced continuous latent representation, $x(s)$ is the binarized latent code, and $\hat{f}(x)$ is the QUBO surrogate of protein fitness.

\subsection{Problem Setup}

Let $\mathcal{D}=\{(s_i,y_i)\}_{i=1}^{N}$ denote a protein fitness dataset, where $s_i$ is a protein sequence and $y_i \in \mathbb{R}$ is its experimentally measured fitness. Our goal is to learn a surrogate model that maps each sequence to a compact binary latent representation and predicts its fitness, while also enabling efficient combinatorial optimization over the latent space.

Unlike conventional approaches that optimize directly in sequence space or in continuous embedding space, we aim to transform protein fitness prediction into a discrete optimization problem. This is motivated by two observations: first, protein sequences are inherently discrete objects; second, binary latent spaces admit efficient combinatorial search and can be directly mapped to QUBO and Ising formulations.

\subsection{Protein Language Model Embeddings}

For each protein sequence $s_i$, we obtain a continuous embedding using a pretrained protein language model. In our experiments, we use ESM-based sequence embeddings \cite{ESM-2, ESM-3} due to their strong empirical performance and broad adoption in protein representation learning. Given a sequence $s_i$ of length $L_i$, the language model produces contextualized residue-level representations $H_i \in \mathbb{R}^{L_i \times d}$, where $d$ is the hidden dimension of the model. To obtain a fixed-length sequence representation, we apply mean pooling across residues:
\[
e_i = \frac{1}{L_i} \sum_{j=1}^{L_i} H_i^{(j)} \in \mathbb{R}^{d}.
\]
This yields one dense embedding vector per sequence.

The role of the protein language model in \modelname is not to directly predict fitness, but rather to provide a biologically informed continuous representation that can later be compressed into a discrete latent code suitable for combinatorial optimization.

\subsection{Continuous-to-Binary Latent Mapping}

A key component of \modelname is the transformation of continuous protein embeddings into compact binary latent representations. While simple projection and thresholding provide a lightweight baseline, more expressive representations can be obtained by learning the latent mapping in a data-driven manner. In this work, we consider a unified family of approaches for constructing binary latent codes, ranging from linear projections to learned autoencoding models.

\paragraph{Linear projection and binarization.}
Given a protein embedding $e \in \mathbb{R}^d$, we first consider linear dimensionality reduction methods such as random projection or principal component analysis (PCA):
\begin{equation}
z = W e, \quad W \in \mathbb{R}^{m \times d},
\end{equation}
where $m \ll d$ is the latent dimension. The resulting continuous latent vector $z \in \mathbb{R}^m$ is then binarized using a thresholding function:
\begin{equation}
x_k = \mathbb{I}(z_k > \tau_k),
\end{equation}
where $\tau_k$ is chosen as the median of the $k$-th projected component. In the reported projection-based experiments, these medians are computed once over each sampled embedding set before the downstream supervised splits; this representation-construction step uses the embeddings only and does not use fitness labels. This approach provides a simple and efficient baseline for constructing binary latent representations.

\paragraph{Deterministic binary autoencoder.}
To learn a more structured latent representation, we introduce a deterministic autoencoder \cite{doi:10.1126/science.1127647, 10.5555/1756006.1953039} that maps protein embeddings into a low-dimensional latent space and reconstructs them:
\begin{equation}
z = f_\theta(e), \quad \hat{e} = g_\phi(z),
\end{equation}
where $f_\theta$ and $g_\phi$ denote the encoder and decoder networks, respectively. To obtain binary latent codes, we apply a binarization function:
\begin{equation}
x = \mathrm{bin}(z),
\end{equation}
where $\mathrm{bin}(\cdot)$ may be implemented using thresholding or a sign function. The model is trained to minimize reconstruction loss:
\begin{equation}
\mathcal{L}_{\text{AE}} = \|e - \hat{e}\|_2^2.
\end{equation}
This formulation allows the latent space to adapt to the structure of protein embeddings, potentially leading to more meaningful and optimization-friendly representations.

\paragraph{Variational latent representations.}
We further consider a probabilistic formulation based on variational autoencoders (VAEs) \cite{kingma2013auto}, which model a distribution over latent variables:
\begin{equation}
z \sim q_\theta(z \mid e), \quad \hat{e} \sim p_\phi(e \mid z).
\end{equation}
The model is trained by maximizing the evidence lower bound (ELBO):
\begin{equation}
\mathcal{L}_{\text{VAE}} = \mathbb{E}_{q_\theta(z \mid e)}[\log p_\phi(e \mid z)] - \mathrm{KL}(q_\theta(z \mid e)\,\|\,p(z)).
\end{equation}
Binary latent codes can be obtained by applying a threshold or sampling from a Bernoulli distribution parameterized by the latent variables:
\begin{equation}
x_k \sim \mathrm{Bernoulli}(\sigma(z_k)).
\end{equation}
This probabilistic formulation enables modeling uncertainty in the latent representation and provides a principled way to explore the latent space.

\paragraph{Discussion.}
These approaches define different ways of constructing binary latent spaces, ranging from simple linear projections to learned deterministic and probabilistic mappings. Importantly, the choice of latent representation directly affects the structure of the induced QUBO model and, consequently, the geometry of the resulting optimization landscape. In Section~\ref{sec:experiments}, we empirically investigate how these different constructions influence both predictive performance and optimization behavior.

\revised{
\subsubsection{Representation-induced search geometry}

A binary encoding determines not only the coordinates used to represent a protein variant, but also the neighborhood structure explored by a combinatorial optimizer. We therefore distinguish between a mere change of coordinates and a transformation that genuinely changes the induced search landscape. If two encodings differ only by a symmetry of the Boolean hypercube, then their one-bit neighborhoods and local optima are unchanged, and the transformation is only a reparameterization. Proposition \ref{prop:search_equivalance} characterizes this search-preserving case.

\begin{proposition}[\label{prop:search_equivalance}Search equivalence under Hamming isometries]
Let
\[
q(x)
=
c+\sum_{i=1}^{m}h_i x_i
+\sum_{1\leq i<j\leq m}J_{ij}x_i x_j,
\qquad x\in\{0,1\}^{m}.
\]
Let \(T:\{0,1\}^{m}\rightarrow\{0,1\}^{m}\) be defined by
\[
T(x)_k=x_{\pi(k)}\oplus b_k,
\]
where \(\pi\) is a permutation and \(b\in\{0,1\}^{m}\).
Then:

\begin{enumerate}
    \item \(T\) preserves Hamming distance;
    \item \(q\circ T^{-1}\) is also a QUBO;
    \item \(x^\star\) is a one-bit local maximizer of \(q\)
    if and only if \(T(x^\star)\) is a one-bit local maximizer
    of \(q\circ T^{-1}\).
\end{enumerate}

Therefore, coordinate permutations and bit complements constitute \textbf{genuine search-preserving reparameterizations} of the Boolean landscape: they change the coordinates but not the Hamming-neighborhood graph or the set of one-bit local optima.
\end{proposition}

\begin{proof}
Coordinate permutations and coordinate-wise complements preserve
whether two codes differ in exactly one bit and hence preserve
Hamming distance. Substituting \(x_i=y_j\) or \(x_i=1-y_j\)
into a degree-two pseudo-Boolean polynomial produces another
polynomial of degree at most two. Finally, \(T\) is a graph
isomorphism of the Hamming cube, so all objective comparisons
between a vertex and its one-bit neighbors are preserved.
\end{proof}

Proposition \ref{prop:search_equivalance} provides a precise baseline for what should be regarded as a mere reparameterization. If a binary representation differs from another only through coordinate permutations or bit complements, then a one-bit optimizer encounters an equivalent search problem. In contrast, dimensionality reduction followed by binarization is not generally restricted to such Hamming isometries: it can change which protein variants are adjacent, introduce collisions between variants, and alter the local-optimum structure of the learned landscape. Thus, comparable pointwise predictive accuracy does not by itself imply equivalent optimization behavior.

\begin{figure*}
\centering
\begin{tikzpicture}[
  vertex/.style={draw,rounded corners,minimum width=1.45cm,
                 minimum height=0.75cm,align=center,font=\small},
  localmax/.style={vertex,double,very thick},
  edge/.style={thick}
]
\node[vertex]   (a00) at (0,0)   {$00$\\$s_1:0$};
\node[localmax] (a01) at (0,2)   {$01$\\$s_2:3$};
\node[localmax] (a10) at (2.4,0) {$10$\\$s_3:2$};
\node[vertex]   (a11) at (2.4,2) {$11$\\$s_4:1$};
\draw[edge] (a00)--(a01)--(a11)--(a10)--(a00);
\node at (1.2,2.85) {Encoding $\phi$};
\node[align=center,font=\small] at (1.2,-0.95)
{$q_\phi(x)=2x_1+3x_2-4x_1x_2$};

\node[vertex]   (b00) at (5.5,0)   {$00$\\$s_1:0$};
\node[localmax] (b01) at (5.5,2)   {$01$\\$s_2:3$};
\node[vertex]   (b10) at (7.9,0)   {$10$\\$s_4:1$};
\node[vertex]   (b11) at (7.9,2)   {$11$\\$s_3:2$};
\draw[edge] (b00)--(b01)--(b11)--(b10)--(b00);
\node at (6.7,2.85) {Encoding $\psi$};
\node[align=center,font=\small] at (6.7,-0.95)
{$q_\psi(x)=x_1+3x_2-2x_1x_2$};
\end{tikzpicture}
\caption{\revised{Prediction-equivalent encodings need not be optimization-equivalent. The same four protein variants \(s_1,\ldots,s_4\), with fitness values \(0,3,2,1\), are assigned to the vertices of a two-dimensional Boolean hypercube under two different binary encodings, \(\phi\) (left) and \(\psi\) (right). Both fitness assignments are represented exactly by a QUBO, so their pointwise predictions on the four variants are identical. However, the change of encoding alters Hamming-1 adjacency: under \(\phi\), both \(s_2\) and \(s_3\) are one-bit local maxima, whereas under \(\psi\), only \(s_2\) is a one-bit local maximum. Double borders denote local maxima. The example demonstrates that exact pointwise agreement does not guarantee equivalent combinatorial search geometry unless the encodings are related by a Hamming isometry as characterized in Proposition \ref{prop:search_equivalance}.}}
\label{fig:encoding_search_geometry}
\end{figure*}

\begin{example}[Prediction-equivalent but search-inequivalent encodings] Consider four protein variants with fitness values

$$ f(s_1)=0,\qquad f(s_2)=3,\qquad f(s_3)=2,\qquad f(s_4)=1. $$

Under the encoding

$$ \phi(s_1)=00,\quad \phi(s_2)=01,\quad \phi(s_3)=10,\quad \phi(s_4)=11, $$

these values are represented exactly by

$$ q_\phi(x)=2x_1+3x_2-4x_1x_2. $$

Both \(01\) and \(10\) are one-bit local maxima. Now consider the alternative encoding

$$ \psi(s_1)=00,\quad \psi(s_2)=01,\quad \psi(s_3)=11,\quad \psi(s_4)=10, $$

for which the same four fitness values are represented exactly by

$$ q_\psi(x)=x_1+3x_2-2x_1x_2. $$

Here, only \(01\) is a one-bit local maximum. Thus, the two representations have identical pointwise fitness values but different Hamming-neighborhood structure and different local optima. The transformation from \(\phi\) to \(\psi\) is therefore not a Hamming-isometric reparameterization of the type characterized in Proposition \ref{prop:search_equivalance}. Figure \ref{fig:encoding_search_geometry} illustrates this construction.
\end{example}

}

\subsection{QUBO Surrogate for Protein Fitness}

We model protein fitness in binary latent space with a QUBO surrogate. Given a binary latent code $x \in \{0,1\}^{m}$, the predicted fitness is:
\[
\hat{f}(x) = \sum_{k=1}^{m} h_k x_k + \sum_{1 \leq k < \ell \leq m} J_{k\ell} x_k x_\ell,
\]
where $h_k \in \mathbb{R}$ captures the unary contribution of latent bit $k$, and $J_{k\ell} \in \mathbb{R}$ captures the pairwise interaction between bits $k$ and $\ell$. Equivalently, the model can be written in matrix form as:
\[
\hat{f}(x) = h^\top x + \frac{1}{2}x^\top J x,
\]
where $J \in \mathbb{R}^{m \times m}$ is the symmetric Hamiltonian (i.e. pairwise interaction matrix) with zero diagonal, and $h \in \mathbb{R}^m$ corresponds to the bias term. This representation connects the latent fitness model to the classical QUBO formulation.

\paragraph{Feature construction.}
For each binary code $x_i$, we construct a feature vector consisting of:
\begin{itemize}
    \item all linear terms $\{x_{ik}\}_{k=1}^{m}$,
    \item all pairwise interaction terms $\{x_{ik}x_{i\ell}\}_{1 \leq k < \ell \leq m}$.
\end{itemize}
The total number of features is therefore:
\[
m + \frac{m(m-1)}{2}.
\]

\paragraph{Parameter estimation.}
We fit the QUBO surrogate using ridge regression. Let $\Phi(X)$ denote the design matrix formed by the linear and pairwise features of the training binary codes, and let $y$ denote the corresponding fitness vector. We solve:
\[
w^\star = \arg\min_{w} \|\Phi(X)w - y\|_2^2 + \lambda \|w\|_2^2,
\]
where $\lambda > 0$ is an $\ell_2$ regularization coefficient. The learned parameter vector $w^\star$ is then unpacked into the unary coefficients $h$ and pairwise coefficients $J$.

This surrogate is attractive for three reasons. First, it is interpretable, since each term corresponds to a unary or pairwise latent effect. Second, it is computationally efficient to fit and evaluate for moderate latent dimensions. Third, it directly yields a QUBO objective suitable for classical and quantum combinatorial optimization.

\revised{
\subsubsection{Statistical-physics interpretation}
The QUBO formulation also admits a direct interpretation as an Ising
energy model. Let
\[
\sigma_i = 2x_i-1 \in \{-1,+1\},
\qquad
x_i=\frac{\sigma_i+1}{2}.
\]
Then the QUBO fitness
\[
\hat f(x)
=
\sum_i h_i x_i
+
\sum_{i<j}J_{ij}x_i x_j
\]
can be rewritten as
\[
\hat f(\sigma)
=
C+\sum_i H_i\sigma_i+
\sum_{i<j}K_{ij}\sigma_i\sigma_j,
\]
where
\[
K_{ij}=\frac{J_{ij}}{4},
\qquad
H_i=
\frac{h_i}{2}
+
\frac{1}{4}\sum_{j\neq i}J_{ij},
\]
and $C$ is independent of $\sigma$. Equivalently,
$E(\sigma)=-\hat f(\sigma)$ is an Ising energy.

Under this interpretation, each binary latent coordinate acts as a
coarse-grained spin, $H_i$ is its effective external field, and
$K_{ij}$ describes the coupling between two latent modes. Positive
couplings favor aligned latent states, whereas negative couplings favor
opposing states. Competing couplings can generate frustration and
multiple local optima, providing a statistical-physics interpretation
of ruggedness in the learned fitness landscape.

Importantly, these latent spins and couplings should not be interpreted
as individual amino-acid states or direct physical residue--residue
interactions. Rather, they describe collective modes of variation in
the protein-language-model representation. Thus, the Ising analogy
provides a coarse-grained description of context-dependent interactions
in the learned protein fitness landscape.

\paragraph{Illustrative example.}
Consider three latent spins and suppose the effective field acting on
the first coordinate is
\[
\mathcal{H}_1(\sigma)
=
H_1+K_{12}\sigma_2+K_{13}\sigma_3,
\]
with
\[
H_1=0.2,\qquad K_{12}=0.8,\qquad K_{13}=-0.7.
\]
If $(\sigma_2,\sigma_3)=(+1,-1)$, then
\[
\mathcal{H}_1=0.2+0.8+0.7=1.7,
\]
so the landscape favors $\sigma_1=+1$. In contrast, if
$(\sigma_2,\sigma_3)=(-1,+1)$, then
\[
\mathcal{H}_1=0.2-0.8-0.7=-1.3,
\]
so the preferred state becomes $\sigma_1=-1$.
Thus, the effect of changing one latent coordinate depends on the
configuration of the others. This context dependence is the latent-space
analogue of epistatic interaction and is precisely the type of structure
captured by the pairwise QUBO terms.
}

\subsection{Latent Space Optimization}

Once the QUBO surrogate has been fitted, we seek binary latent codes that maximize the predicted fitness:
\[
x^\star = \arg\max_{x \in \{0,1\}^{m}} \hat{f}(x).
\]
Because this is a discrete combinatorial problem, we employ search strategies that operate directly in the binary latent space.

\paragraph{Simulated annealing \cite{kirkpatrick1983optimization}.}
Simulated annealing starts from an initial latent code and iteratively proposes single-bit flips. Moves that improve the objective are accepted, while worse moves may be accepted with a temperature-dependent probability, enabling escape from local optima. The temperature is gradually reduced during the search.

\paragraph{Genetic algorithm \cite{10.7551/mitpress/1090.001.0001}.}
The genetic algorithm maintains a population of binary latent codes and evolves them through selection, crossover, and mutation. This allows broader exploration of the latent space and is particularly useful in higher-dimensional settings where multiple promising regions may exist.

\paragraph{Baselines.}
To contextualize performance, we also compare against greedy hill climbing, random search, and a lightweight latent Bayesian-style search baseline. Greedy hill climbing iteratively flips the single bit that most improves the objective until convergence. Random search samples binary codes uniformly and retains the best candidate. The latent Bayesian-style method uses a kernel-based uncertainty heuristic over binary latent codes.

\subsection{Decoding of Optimized Latent Representations}

A key step in \modelname is mapping optimized binary latent codes back to valid protein sequences. Since optimization is performed in latent space, decoding provides the connection between combinatorial search and biologically meaningful sequences.

Retrieval-based decoding. As a conservative baseline, we first use nearest-neighbor retrieval in Hamming space. Given an optimized latent code, we identify the closest latent codes from the training set and return their associated sequences. This ensures that decoded sequences remain within the support of the observed data distribution.

Neural decoding. To enable sequence generation beyond the observed dataset, we introduce a latent-conditioned mutation decoder. Instead of generating sequences autoregressively, the decoder predicts mutation patterns relative to a wild-type sequence. For each position, the model predicts (i) whether a mutation occurs and (ii) the identity of the mutated amino acid. This formulation leverages the structure of deep mutational scanning datasets, where sequences differ from a common backbone.

Discussion. We find that decoding performance strongly depends on the structure of the latent space. In particular, PCA-based binary representations produce significantly more decodable latent codes than both random projections and AE/VAE representations. These results emphasize that representation learning must be aligned with both optimization and decoding objectives.

\subsection{Quantum Annealing Compatibility}

An important property of \modelname is that its optimization objective is already expressed in QUBO form. This means the learned latent fitness landscape can be optimized not only with classical combinatorial solvers, but also with quantum annealing hardware or other Ising/QUBO solvers without changing the model formulation. In this sense, \modelname provides a quantum-compatible interface between protein representation learning and discrete optimization.

In the present work, we focus on classical solvers to establish the feasibility of the framework and to analyze the optimization behavior in binary latent space. Nevertheless, the QUBO formulation opens a clear pathway toward future quantum-assisted protein engineering, where learned latent fitness landscapes could be deployed on quantum annealers or related specialized optimization hardware.

\subsection{Computational Complexity}

Let $m$ denote the binary latent dimension. The QUBO surrogate uses $
m + \frac{m(m-1)}{2} = \mathcal{O}(m^2)$ features. Therefore, increasing latent dimensionality increases expressive power but also increases both the number of surrogate parameters and the difficulty of combinatorial search. This trade-off motivates our experimental study over multiple latent dimensions.

By separating the dense protein language model embedding stage from the binary optimization stage, \modelname keeps the most expensive neural computation fixed while enabling fast repeated optimization in the compact binary latent space. This makes the framework computationally attractive for studying representation--optimization trade-offs.
\section{Results and Discussion} \label{sec:experiments}


\revised{
Our experimental study evaluates three aspects of \modelname: (i) the predictive quality of the internal QUBO surrogate against experimentally measured fitness; (ii) the effect of binary representation on decoding and combinatorial optimization; and (iii) model-based evaluation of decoded candidates for which experimental measurements are unavailable. Accordingly, Sections \ref{subsec:qubo_surrogate} and \ref{subsec:retrieval_benchmark} provide experimentally grounded evaluations, whereas Section \ref{subsec:design_benchmark} provides secondary model-based candidate prioritization.
}

\subsection{Datasets and Data-Scarce Regimes}
\label{subsec:datasets}

\begin{table*}
\centering
\caption{ProteinGym datasets used in the experiments. 
We consider two representative protein fitness landscapes: GFP for fluorescence-based fitness and AAV capsid for viral fitness. 
For each dataset, we construct multiple subset sizes to study data-scarce and moderate-data regimes.}
\label{tab:datasets}
\begin{tabular}{lcccc}
\toprule
Dataset & Task & WT length & Full library size & Subset sizes \\
\midrule
GFP & Fluorescence / fitness & 238 & 51{,}714 & 1000, 2000, 5000, 10000 \\
AAV & Capsid / fitness & 735 & 42{,}328 & 1000, 2000, 5000, 10000 \\
\bottomrule
\end{tabular}
\end{table*}

We evaluate \modelname on protein fitness landscapes from the ProteinGym benchmark \cite{ProteinGym}. Our primary experiments focus on two representative tasks: GFP and AAV. These two datasets provide complementary settings for studying both mutation-effect prediction and optimization over rugged protein fitness landscapes. For each dataset, we construct subset sizes of \(\{1000, 2000, 5000, 10000\}\) variants (see Table \ref{tab:datasets}). We treat the \(1000\) and \(2000\) settings as \emph{data-scarce regimes}, and the \(5000\) and \(10000\) settings as \emph{moderate-data regimes}. Unless otherwise specified, each subset is split into training, validation, and test sets using a stratified split based on fitness quantiles, with the validation split used for model selection and the test split reserved for final reporting.
Stochastic optimization results are aggregated over five independent runs using seeds 0, 1, 2, 3, and 4, whereas model-fitting components use the fixed seeds specified in Appendix~\ref{sec:additional_experiments} unless otherwise noted.

\revised{
Data splitting and random-seed protocols are component-specific and are reported in Appendix \ref{sec:additional_experiments}. The sequence-level evaluator uses a fixed 70\%/10\%/20\% train/validation/test split, while the internal QUBO uses its separately specified 80\%/20\% split. Stochastic optimization results are aggregated over multiple runs, whereas model-fitting components use the fixed seeds stated in Appendix \ref{sec:additional_experiments}.
}

\subsection{External Sequence-Level Fitness Oracle}
\label{subsec:oracle}

\revised{
For the secondary evaluation of decoded sequences that lack experimental measurements, we use a separate sequence-level fitness predictor that maps a protein sequence to a scalar fitness estimate.
}
This oracle is conceptually distinct from the internal \modelname QUBO surrogate. The external oracle serves two purposes: 
(i) it provides a standardized black-box evaluator for comparing protein design methods that output sequences, and (ii) it enables evaluation of generated sequences that do not exactly match an experimentally observed variant.

For each protein sequence, we compute an ESM-based sequence embedding by mean pooling residue-level representations. On top of these embeddings, we train a family of classical regression models, including ridge regression, gradient-boosted decision trees (XGBoost), and Gaussian process regression. All models are trained on the same embedding space to isolate the effect of the regression model rather than the upstream representation. Model selection is performed using a validation split, and the best-performing oracle for each dataset-size regime is used in the downstream sequence-level optimization experiments.

We evaluate the external oracle using standard regression metrics, including Spearman correlation, Pearson correlation, root mean squared error (RMSE), and mean absolute error (MAE). Table~\ref{tab:oracle} summarizes the results across both GFP and AAV datasets, and Figure~\ref{fig:oracle_gfp_aav} visualizes the scaling behavior.

Across both datasets, we observe a consistent and systematic dependence of oracle performance on the amount of available data. In both low-data regimes (1000--2000 samples) and moderate-data regimes (5000--10000 samples), Gaussian process regression achieves the strongest performance, reflecting its ability to model uncertainty and adapt to limited supervision. As the dataset size increases, ridge regression becomes increasingly competitive and often achieves near the best performance, suggesting that the fitness signal is largely captured by a linear model in the ESM embedding space when sufficient data is available. In contrast, XGBoost underperforms in the smallest-data regime, indicating overfitting, but improves steadily with increasing dataset size. Based on these results, we use Gaussian process oracles in data-scarce settings when computationally feasible, and rely on ridge regression and XGBoost as scalable sequence-level evaluators in larger regimes.

\begin{table*}[t]
\centering
\caption{Performance of external sequence-level fitness oracles built on top of ESM embeddings across GFP and AAV datasets. 
Results are reported on the test set. The best Spearman correlation for each dataset and training size is highlighted in bold. \revised{Gaussian-process regression achieves the highest test Spearman correlation in seven of the eight dataset-size settings, while Ridge Regression performs best for AAV with 1,000 samples. Thus, Gaussian-process regression is generally strongest, but no evaluator uniformly dominates every regime.}}
\label{tab:oracle}
\begin{tabular}{llccccc}
\toprule
Dataset & Train size & Oracle model & Spearman $\uparrow$ & Pearson $\uparrow$ & RMSE $\downarrow$ & MAE $\downarrow$ \\
\midrule
\multirow{12}{*}{GFP} & \multirow{3}{*}{1000}  & Gaussian Process & \textbf{0.657} & \textbf{0.682} & \textbf{0.759} & \textbf{0.612} \\
 &  & Ridge Regression & 0.637 & 0.665 & 0.859 & 0.652 \\
 &  & XGBoost          & 0.520 & 0.557 & 0.863 & 0.711 \\
\cmidrule{2-7}
 & \multirow{3}{*}{2000}  & Gaussian Process & \textbf{0.714} & \textbf{0.729} & \textbf{0.742} & \textbf{0.586} \\
 &  & Ridge Regression & 0.660 & 0.669 & 0.817 & 0.637 \\
 &  & XGBoost          & 0.650 & 0.675 & 0.802 & 0.624 \\
\cmidrule{2-7}
 & \multirow{3}{*}{5000}  & Gaussian Process & \textbf{0.729} & \textbf{0.754} & \textbf{0.699} & \textbf{0.542} \\
 &  & Ridge Regression & 0.702 & 0.697 & 0.772 & 0.607 \\
 &  & XGBoost          & 0.635 & 0.685 & 0.772 & 0.596 \\
\cmidrule{2-7}
 & \multirow{3}{*}{10000} & Gaussian Process & \textbf{0.783} & \textbf{0.796} & \textbf{0.646} & \textbf{0.495} \\
 &  & Ridge Regression & 0.735 & 0.730 & 0.730 & 0.580 \\
 &  & XGBoost          & 0.693 & 0.729 & 0.731 & 0.568 \\
\midrule
 \multirow{12}{*}{AAV} & \multirow{3}{*}{1000}  & Gaussian Process & 0.671 & 0.650 & \textbf{2.390} & \textbf{1.906} \\
 &  & Ridge Regression & \textbf{0.706} & \textbf{0.656} & 2.618 & 2.058 \\
 &  & XGBoost          & 0.611 & 0.607 & 2.501 & 2.005 \\
\cmidrule{2-7}
 & \multirow{3}{*}{2000}  & Gaussian Process & \textbf{0.792} & \textbf{0.762} & \textbf{2.048} & \textbf{1.629} \\
 &  & Ridge Regression & 0.779 & 0.751 & 2.155 & 1.690 \\
 &  & XGBoost          & 0.694 & 0.694 & 2.255 & 1.813 \\
\cmidrule{2-7}
 & \multirow{3}{*}{5000}  & Gaussian Process & \textbf{0.806} & \textbf{0.782} & \textbf{1.898} & \textbf{1.505} \\
 &  & Ridge Regression & 0.796 & 0.763 & 1.962 & 1.547 \\
 &  & XGBoost          & 0.729 & 0.724 & 2.094 & 1.659 \\
\cmidrule{2-7}
 & \multirow{3}{*}{10000} & Gaussian Process & \textbf{0.838} & \textbf{0.821} & \textbf{1.712} & \textbf{1.346} \\
 &  & Ridge Regression & 0.808 & 0.781 & 1.874 & 1.489 \\
 &  & XGBoost          & 0.777 & 0.772 & 1.914 & 1.499 \\
\bottomrule
\end{tabular}
\end{table*}

\begin{figure*}[t]
    \centering
    \includegraphics[width=0.9\linewidth]{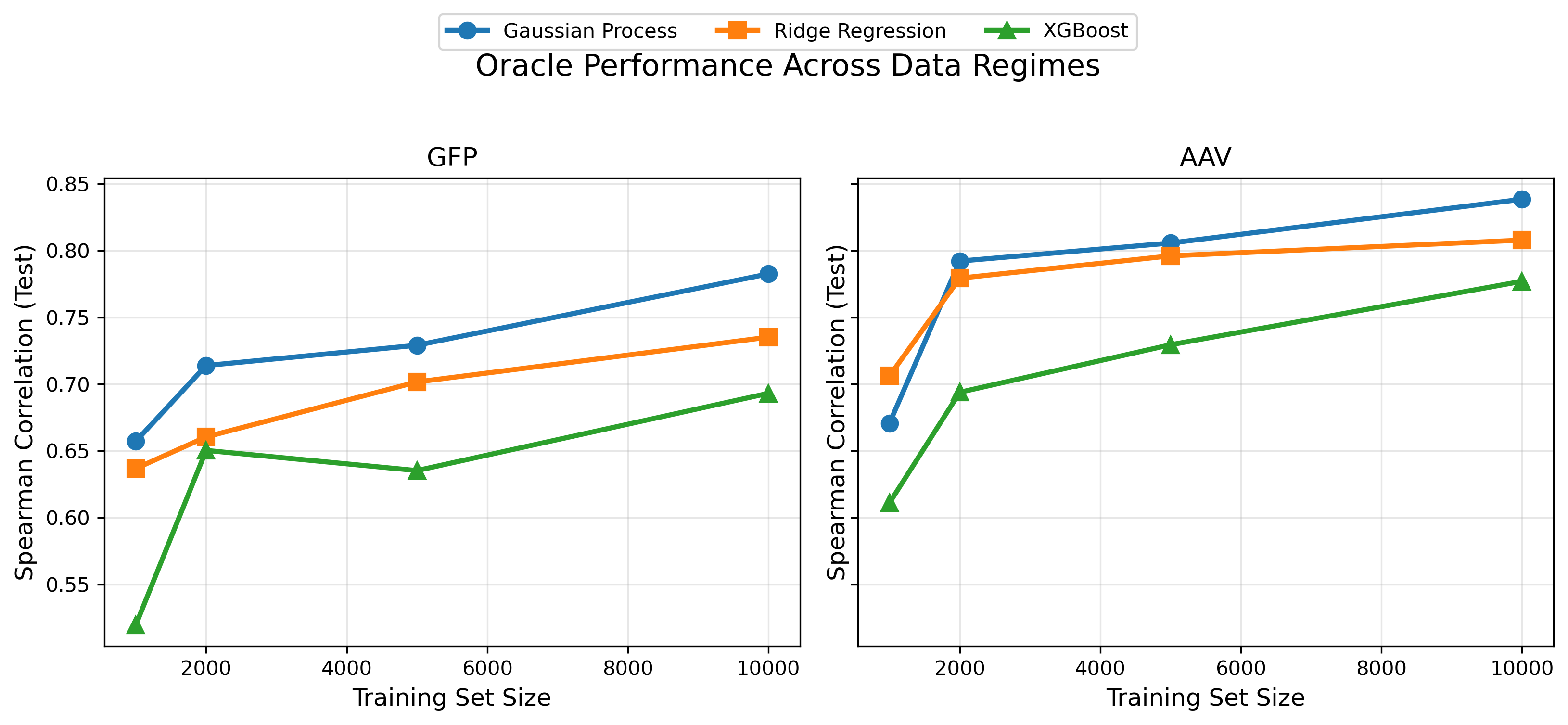}
    \caption{Test-set performance of external sequence-level fitness oracles across data regimes on GFP and AAV. Each panel shows the Spearman correlation between predicted and ground-truth fitness as a function of the number of training samples. Across both datasets, Gaussian process regression often achieves the strongest performance in both low-data and moderate-data regimes, reflecting its ability to model uncertainty under limited supervision. As the dataset size increases, ridge regression becomes increasingly competitive and often achieves near the best performance, indicating that the fitness signal is largely captured by a linear model in the ESM embedding space. In contrast, XGBoost underperforms in the smallest-data regime, suggesting overfitting, but improves steadily with additional data. 
    \revised{These results establish the predictive performance of the sequence-level evaluator on held-out measured variants and motivate its use for secondary model-based scoring of decoded sequences, while not constituting experimental validation of newly generated candidates.}
    }
    \label{fig:oracle_gfp_aav}
\end{figure*}

\subsection{Latent Representation Learning and Decoder Analysis}
\label{subsec:repr_decoder}

We compare four families of latent representations: (i) random projection with thresholding, (ii) PCA with thresholding, (iii) deterministic autoencoder (AE), and (iv) variational autoencoder (VAE). All methods operate on the same ESM embeddings to isolate the effect of latent representation.

\paragraph{Latent quality and collapse of learned representations.}
We first examine whether learned latent models produce usable binary codes after binarization. Table~\ref{tab:latent_collapse} reports reconstruction error, bit entropy, and the number of active latent dimensions. Although AE and VAE achieve very low reconstruction error, their binarized latent codes exhibit near-zero entropy and essentially no active dimensions. This indicates a collapse of the binary latent space, where most bits become constant across samples and fail to encode meaningful variation. As a result, these representations do not provide a useful combinatorial search space for downstream optimization.

\paragraph{Decoding performance of projection-based representations.}
We next evaluate the decodability of binary latent codes using a mutation-conditioned decoder. Table~\ref{tab:decoder_quality} shows that projection-based representations, particularly PCA, significantly outperform random projections in both mutation F1 and mutated-residue accuracy. For example, on AAV with 10000 training samples, PCA with 64 latent dimensions achieves substantially higher decoding accuracy than random projection at the same dimensionality. Similar trends hold in the data-scarce regime on GFP.

\paragraph{Scaling behavior with data and latent dimension.}
Figure~\ref{fig:decoder_scaling} further illustrates how decoding performance scales with the number of training samples at a fixed latent dimension of 64 bits. Across both GFP and AAV datasets, PCA consistently achieves higher mutation F1 than random projection, and performance improves as more data becomes available. This indicates that PCA not only preserves meaningful structure but also benefits from additional supervision.

\paragraph{Discussion.}
Taken together, these results reveal a clear mismatch between reconstruction quality and optimization readiness. While AE and VAE reconstruct embeddings accurately, their binarized latent spaces collapse and fail to support decoding or combinatorial search. In contrast, simple structured representations such as PCA maintain both variability and decodability, leading to substantially better downstream performance. These findings motivate our focus on PCA-based binary latent representations in the end-to-end design benchmark. These observations also suggest that latent representations for protein design should be evaluated not only by reconstruction or predictive accuracy, but by their ability to support both decoding and combinatorial optimization.

\revised{
The advantage of PCA should not be interpreted as a general superiority of linear over nonlinear representation learning. Under our protocol, PCA combines a data-adaptive, variance-ordered orthogonal basis with per-coordinate median thresholding. In the absence of ties, median thresholding makes each marginal bit approximately balanced and therefore maximizes its empirical Bernoulli entropy. Random projection shares this calibration step but not PCA’s data-adaptive variance ordering. By contrast, the AE/VAE baselines optimize continuous reconstruction or ELBO objectives and are binarized post hoc at zero. These objectives do not directly optimize bit balance, Hamming geometry, or post-threshold decoding. The observed collapse therefore demonstrates an objective-and-calibration mismatch between continuous representation learning and binary search, rather than an intrinsic limitation of neural encoders. Quantization-aware or entropy-regularized nonlinear encoders are a natural direction for future work.
}

\begin{table*}[t]
\centering
\caption{Collapse of learned latent representations after binarization on GFP with 1000 training samples. Although AE and VAE achieve low reconstruction error, their binarized latent codes exhibit \revised{near-zero entropy and essentially no active dimensions}.}
\label{tab:latent_collapse}
\begin{tabular}{l c c c c}
\toprule
Representation & Latent dim & Bit entropy $\uparrow$ & Active dims $\uparrow$ & Recon. MSE $\downarrow$ \\
\midrule
\multirow{2}{*}{AE}& 8 & 0.000 & 0.00 & 0.000125 \\
 & 64 & 0.000 & 0.00 & 0.000124 \\
\midrule
\multirow{2}{*}{VAE} & 8 & 0.149 & 0.25 & 0.000126 \\
 & 64 & 0.002 & 0.00 & 0.000164 \\
\bottomrule
\end{tabular}
\end{table*}

\begin{table*}[t]
\centering
\caption{Decoding quality of projection-based binary latent representations. PCA consistently achieves higher mutation F1 and mutated-residue accuracy than random projection, and both metrics improve with increasing latent dimensionality. The better performance between two latent dimensions (i.e. bits) of 16 and 64 is highlighted in bold across different settings. The higher latent dimension (i.e. 64) consistently outperforms the lower latent dimension (i.e. 16) in decoding quality.}
\label{tab:decoder_quality}
\begin{tabular}{l c l c c c}
\toprule
Dataset & Train size & Representation & Latent dim & Mutation F1 $\uparrow$ & AA accuracy $\uparrow$ \\
\midrule
\multirow{4}{*}{GFP} & \multirow{4}{*}{1000} & PCA & 16 & 0.110 & 0.551 \\
 &  & PCA & 64 & \textbf{0.273} & \textbf{0.595} \\
\cmidrule{3-6}
 &  & Random & 16 & 0.102 & 0.510 \\
 &  & Random & 64 & \textbf{0.168} & \textbf{0.553} \\
 \midrule
\multirow{4}{*}{AAV} & \multirow{4}{*}{10000} & PCA & 16 & 0.369 & 0.367 \\
 &  & PCA & 64 & \textbf{0.603} & \textbf{0.602} \\
\cmidrule{3-6}
 &  & Random & 16 & 0.338 & 0.316 \\
 &  & Random & 64 & \textbf{0.469} & \textbf{0.475} \\
\bottomrule
\end{tabular}
\end{table*}

\begin{figure*}[t]
\centering
\includegraphics[width=\linewidth]{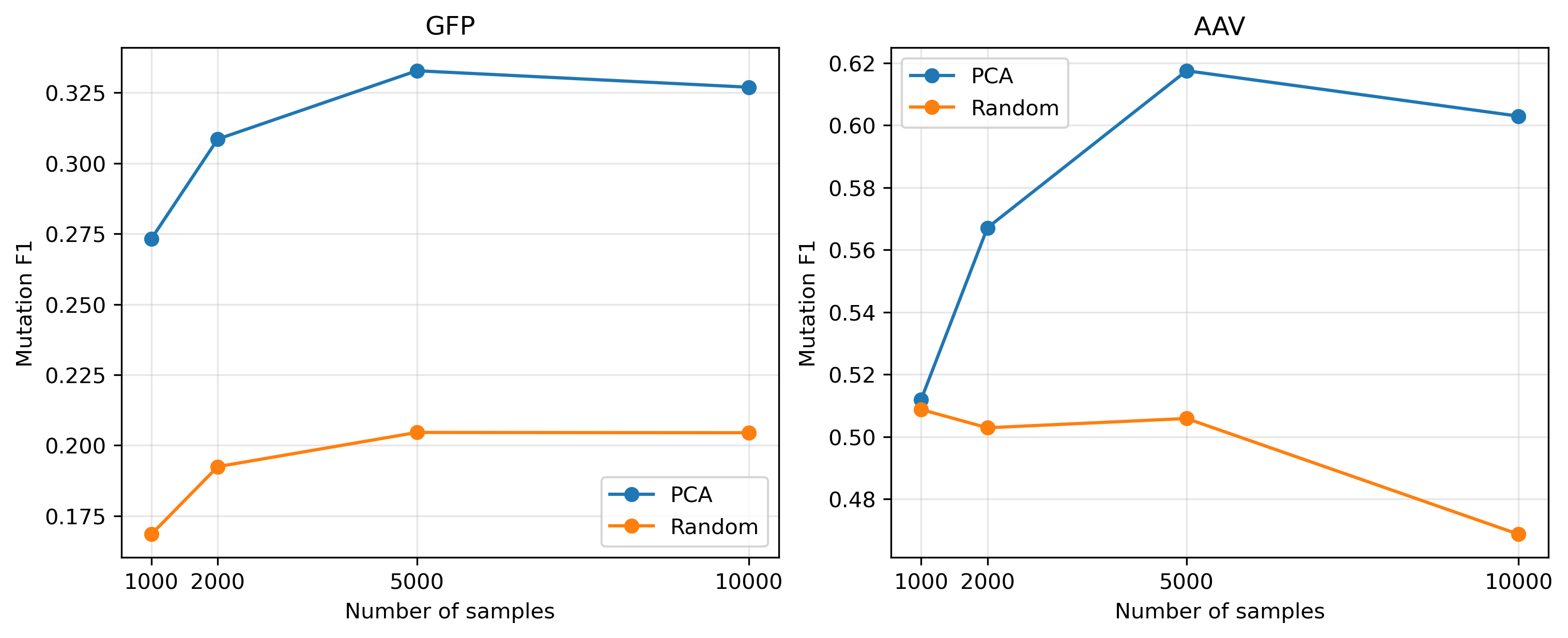}
\caption{Scaling of decoding performance with dataset size for PCA-based and random-projection binary latent representations at 64 bits. 
The x-axis shows the number of training samples and the y-axis shows mutation F1. 
Across both GFP and AAV datasets, PCA consistently achieves higher decoding accuracy than random projection, 
and performance improves with increasing data.}
\label{fig:decoder_scaling}
\end{figure*}

\subsection{Internal QUBO Surrogate Performance}
\label{subsec:qubo_surrogate}

Given a binary latent representation, \modelname fits a quadratic surrogate over unary and pairwise latent features. This surrogate defines the QUBO objective used for combinatorial optimization. In this section, we evaluate how well different latent representations support the construction of such a surrogate.

Table~\ref{tab:qubo_surrogate} reports predictive performance of the learned QUBO models using Spearman correlation on held-out test data. Importantly, this evaluation is not intended to establish state-of-the-art fitness prediction, but rather to assess whether the latent representation induces a meaningful optimization landscape.


\revised{
Across the projection-based representations shown in Table \ref{tab:qubo_surrogate}, PCA and random projection often achieve comparable predictive performance under the QUBO surrogate. In several settings, the differences in Spearman correlation are modest, and neither representation uniformly dominates across all data regimes and latent dimensions.
}


\revised{
However, when considered together with the representation and decoder analysis in Section \ref{subsec:repr_decoder} and the end-to-end benchmark in Section \ref{subsec:design_benchmark}, a clear discrepancy emerges between predictive accuracy and downstream optimization effectiveness. Random projection can yield competitive QUBO predictive performance while exhibiting weaker decoding and design performance than PCA.
}

These findings highlight a key distinction: predictive performance alone does not determine the usefulness of a representation for combinatorial optimization. Instead, the geometry of the latent space and its compatibility with decoding and search play a central role. This observation is consistent with the theoretical analysis in Section~\ref{sec:theory}, which shows that different representations can induce substantially different optimization landscapes even when predictive accuracy is similar.

Overall, these results demonstrate that while the QUBO surrogate can be fitted across a range of representations, only structured latent spaces such as PCA lead to effective end-to-end protein design.

\begin{table*}[t]
\centering
\caption{Predictive performance of the internal QUBO surrogate across latent representations, training data sizes, and latent dimensions. We report held-out test Spearman correlation. Although predictive differences are often modest, downstream optimization behavior differs substantially across representations. The better performance between the two binary-latent representations (i.e. PCA and random projection) is highlighted in bold. 
\revised{
Neither representation uniformly dominates in predictive performance, and the differences are often modest. This contrasts with their substantially different downstream decoding and optimization behavior.
}
}
\label{tab:qubo_surrogate}
\begin{tabular}{l c l c c}
\toprule
Dataset & Train size & Latent dim & Representation & Spearman $\uparrow$ \\
\midrule
\multirow{8}{*}{GFP} & \multirow{4}{*}{1000} & \multirow{2}{*}{16} & Random & \textbf{0.291} \\
 &  &  & PCA & 0.147 \\
\cmidrule{3-5}
 &  & \multirow{2}{*}{64} & Random  & 0.196 \\
 &  &  & PCA & \textbf{0.277} \\
\cmidrule{2-5}
 & \multirow{4}{*}{10000} & \multirow{2}{*}{16} & Random & 0.302 \\
 &  &  & PCA & \textbf{0.303} \\
\cmidrule{3-5}
 &  & \multirow{2}{*}{64} & Random & 0.413 \\
 &  &  & PCA & \textbf{0.532} \\
\midrule
\multirow{8}{*}{AAV} & \multirow{4}{*}{1000} & \multirow{2}{*}{16} & Random & \textbf{0.417} \\
 &  &  & PCA & 0.393 \\
\cmidrule{3-5}
 &  & \multirow{2}{*}{64} & Random & 0.407 \\
 &  &  & PCA & \textbf{0.511} \\
\cmidrule{2-5}
 & \multirow{4}{*}{10000} & \multirow{2}{*}{16} & Random & 0.366 \\
 &  &  & PCA & \textbf{0.422} \\
\cmidrule{3-5}
 &  & \multirow{2}{*}{64} & Random & 0.534 \\
 &  &  & PCA & \textbf{0.564} \\
\bottomrule
\end{tabular}
\end{table*}

\subsection{Conservative Retrieval-Based Optimization Benchmark}
\label{subsec:retrieval_benchmark}

We evaluate \modelname in a conservative setting where optimized latent codes are mapped back to observed sequences via nearest-neighbor retrieval in Hamming space. This benchmark isolates the quality of the induced latent fitness landscape without confounding effects from neural decoding.

To ensure a fair comparison across optimization methods, all results are reported using PCA-based latent representations with a fixed latent dimension of 64 bits. Tables~\ref{tab:retrieval_gfp} and~\ref{tab:retrieval_aav} summarize the mean retrieval-based optimization performance across five runs (seeds 0, 1, 2, 3, and 4) on GFP and AAV, respectively.

Across both datasets and all data regimes, Simulated Annealing (SA), Genetic Algorithms (GA), and Greedy Hill Climb (GHC) consistently achieve strong performance, indicating that the induced QUBO landscape is well-structured for combinatorial search. In particular, these methods reliably identify high-fitness regions of the observed sequence manifold even in data-scarce settings, while latent Bayesian-style optimization (LBO) is generally less competitive.

These results demonstrate that PCA-based binary latent representations produce optimization-friendly landscapes that can be effectively explored by classical combinatorial solvers, providing a strong foundation for the full end-to-end design pipeline.

\begin{table*}[t]
\centering
\caption{Retrieval-based optimization results on GFP using PCA latent representations with 64 bits. Each row reports the mean performance across five runs (seeds 0, 1, 2, 3, and 4) for one optimizer. Optimizers include Simulated Annealing (SA), Genetic Algorithm (GA), Random Search (RS), Greedy Hill Climb (GHC), and Latent Bayesian Optimization (LBO). In each data size, the best performance in each metric is highlighted in bold. \revised{QUBO objective gain denotes the change in the learned surrogate objective relative to the initial code and should not be interpreted as an experimentally measured fitness improvement. Retrieved fitness and percentile are obtained from the measured nearest-neighbor variant.}}
\label{tab:retrieval_gfp}
\resizebox{\textwidth}{!}{%
\begin{tabular}{c l c c c}
\toprule
Train size & Optimizer & \revised{QUBO objective gain} $\uparrow$ & \revised{Retrieved experimental fitness} $\uparrow$ & \revised{Retrieved experimental percentile} $\uparrow$ \\
\midrule
\multirow{5}{*}{1000} & SA & 4.978 & 3.568 & 73.200 \\
 & GA & \textbf{5.044} & 3.440 & 60.925 \\
 & RS & 1.750 & 3.443 & 58.650 \\
 & GHC & 4.918 & \textbf{3.616} & \textbf{74.825} \\
 & LBO & -1.509 & 2.766 & 58.050 \\
\midrule
\multirow{5}{*}{2000} & SA & \textbf{11.529} & 3.558 & 70.700 \\
 & GA & 11.103 & 2.908 & 62.025 \\
 & RS & 4.770 & 2.593 & 43.175 \\
 & GHC & 10.230 & \textbf{3.580} & \textbf{75.013} \\
 & LBO & -1.119 & 2.953 & 47.700 \\
\midrule
\multirow{5}{*}{5000} & SA & \textbf{6.757} & 3.479 & 63.410 \\
 & GA & 6.628 & 3.689 & \textbf{83.330} \\
 & RS & 2.617 & 3.078 & 61.640 \\
 & GHC & 6.083 & \textbf{3.761} & 80.605 \\
 & LBO & -1.005 & 2.775 & 54.085 \\
\midrule
\multirow{5}{*}{10000} & SA & \textbf{4.162} & 3.220 & 69.740 \\
 & GA & 3.938 & \textbf{3.607} & \textbf{74.595} \\
 & RS & 1.609 & 3.590 & 71.938 \\
 & GHC & 3.688 & 3.175 & 64.562 \\
 & LBO & -1.253 & 1.836 & 28.780 \\
\bottomrule
\end{tabular}
}
\end{table*}

\begin{table*}[t]
\centering
\caption{Retrieval-based optimization results on AAV using PCA latent representations with 64 bits. Each row reports the mean performance across five runs (seeds 0, 1, 2, 3, and 4) for one optimizer. Optimizers include Simulated Annealing (SA), Genetic Algorithm (GA), Random Search (RS), Greedy Hill Climb (GHC), and Latent Bayesian Optimization (LBO). In each data size, the best performance in each metric is highlighted in bold. \revised{QUBO objective gain denotes the change in the learned surrogate objective relative to the initial code and should not be interpreted as an experimentally measured fitness improvement. Retrieved fitness and percentile are obtained from the measured nearest-neighbor variant.}}
\label{tab:retrieval_aav}
\resizebox{\textwidth}{!}{%
\begin{tabular}{c l c c c}
\toprule
Train size & Optimizer & \revised{QUBO objective gain} $\uparrow$ & \revised{Retrieved experimental fitness} $\uparrow$ & \revised{Retrieved experimental percentile} $\uparrow$ \\
\midrule
\multirow{5}{*}{1000} & SA & \textbf{11.364} & 3.303 & 88.225 \\
 & GA & 10.342 & \textbf{3.949} & 90.000 \\
 & RS & 1.224 & 1.945 & 77.550 \\
 & GHC & 9.362 & 3.902 & \textbf{91.400} \\
 & LBO & -9.071 & -0.606 & 53.200 \\
\midrule
\multirow{5}{*}{2000} & SA & \textbf{28.756} & 1.794 & 72.275 \\
 & GA & 27.108 & -1.256 & 47.325 \\
 & RS & 8.523 & 0.343 & 66.050 \\
 & GHC & 26.511 & \textbf{1.856} & \textbf{74.388} \\
 & LBO & -11.177 & -0.116 & 61.725 \\
\midrule
\multirow{5}{*}{5000} & SA & \textbf{14.627} & -0.654 & 50.165 \\
 & GA & 13.636 & 3.869 & \textbf{95.325} \\
 & RS & 4.153 & 1.477 & 77.715 \\
 & GHC & 13.162 & \textbf{6.643} & 94.975 \\
 & LBO & -5.378 & -1.871 & 44.250 \\
\midrule
\multirow{5}{*}{10000} & SA & \textbf{12.575} & 3.038 & 86.198 \\
 & GA & 12.321 & \textbf{3.720} & \textbf{92.948} \\
 & RS & 4.115 & 2.278 & 78.420 \\
 & GHC & 11.545 & 3.262 & 91.380 \\
 & LBO & -3.503 & 0.003 & 61.045 \\
\bottomrule
\end{tabular}
}
\end{table*}

\subsection{End-to-End Sequence Design Benchmark}
\label{subsec:design_benchmark}

We evaluate \modelname in an end-to-end protein design setting where optimized binary latent codes are decoded into protein sequences and scored using the external oracle. 
\revised{
This benchmark provides a model-based assessment of whether optimization in binary latent space leads to sequences with higher predicted fitness under the external sequence-level evaluator.
}

Figure~\ref{fig:section56_main} summarizes the best end-to-end performance achieved by PCA-based and random-projection binary latent representations across training regimes on GFP and AAV. For each dataset and train size, we report the strongest configuration after jointly selecting the optimizer and latent dimension.

Across most settings, PCA-based binary latent representations achieve the strongest performance. This advantage is especially pronounced at moderate and larger data regimes, where higher-dimensional PCA latents (32 and 64 bits) consistently produce the highest oracle scores. These findings are consistent with the decoder analysis in Section~\ref{subsec:repr_decoder}, where PCA was shown to be substantially more decodable than both random projections and learned AE/VAE latent spaces.

We also observe that larger latent dimensions generally improve end-to-end performance, particularly for PCA, indicating that structured latent spaces can benefit from increased capacity without sacrificing decodability. In contrast, random projections are less reliable and usually underperform PCA, although a few low-data settings exhibit isolated exceptions.

Among the optimizers, Greedy Hill Climb (GHC), Genetic Algorithm (GA), and Simulated Annealing (SA) are consistently competitive, while latent Bayesian-style search (LBO) is generally weaker. Overall, these results show that \modelname provides an effective framework for end-to-end protein design, where binary latent optimization combined with structured representations \revised{enables the prioritization of candidate protein sequences with high predicted fitness under a fixed oracle evaluation budget.}

\begin{table*}[t]
\centering
\caption{End-to-end protein design benchmark under a fixed external-oracle evaluation budget. For each dataset and train size, we report the best-performing random-projection and PCA-based \modelname configuration after optimization, decoding, and oracle scoring. Optimizers include Simulated Annealing (SA), Genetic Algorithm (GA), Random Search (RS), Greedy Hill Climb (GHC), and Latent Bayesian Optimization (LBO). The better performance between PCA and random projection is highlighted in bold.}
\label{tab:main}
\resizebox{\textwidth}{!}{%
\begin{tabular}{llccccc}
\toprule
Dataset & Train size & Representation & Optimizer & Latent dim & Best score $\uparrow$ & Top-10 mean $\uparrow$ \\
\midrule
\multirow{8}{*}{GFP} & \multirow{2}{*}{1000} & PCA & RS & 64 & 4.119 & 3.903 \\
 &  & Random & GA & 64 & \textbf{4.170} & \textbf{3.312} \\
\cmidrule{2-7}
 & \multirow{2}{*}{2000} & PCA & GA & 64 & \textbf{4.027} & \textbf{3.549} \\
 &  & Random & GHC & 32 & 3.448 & 2.693 \\
\cmidrule{2-7}
 & \multirow{2}{*}{5000} & PCA & GHC & 64 & \textbf{4.270} & \textbf{3.972} \\
 &  & Random & GHC & 32 & 3.815 & 2.228 \\
\cmidrule{2-7}
 & \multirow{2}{*}{10000} & PCA & GHC & 64 & \textbf{5.554} & \textbf{5.042} \\
 &  & Random & GHC & 32 & 4.806 & 3.829 \\
\midrule
\multirow{8}{*}{AAV} & \multirow{2}{*}{1000} & PCA & GHC & 64 & \textbf{5.232} & \textbf{2.962} \\
 &  & Random & GHC & 64 & 4.066 & -1.557 \\
\cmidrule{2-7}
 & \multirow{2}{*}{2000} & PCA & RS & 32 & 2.866 & 0.349 \\
 &  & Random & GHC & 32 & \textbf{4.184} & \textbf{-2.674} \\
\cmidrule{2-7}
 & \multirow{2}{*}{5000} & PCA & SA & 64 & \textbf{3.999} & \textbf{2.249} \\
 &  & Random & GA & 64 & 2.548 & -1.150 \\
\cmidrule{2-7}
 & \multirow{2}{*}{10000} & PCA & GA & 64 & \textbf{4.541} & \textbf{2.091} \\
 &  & Random & LBO & 16 & 1.197 & -4.205 \\
\bottomrule
\end{tabular}
}
\end{table*}

\begin{figure*}[t]
\centering
\includegraphics[width=0.9\textwidth]{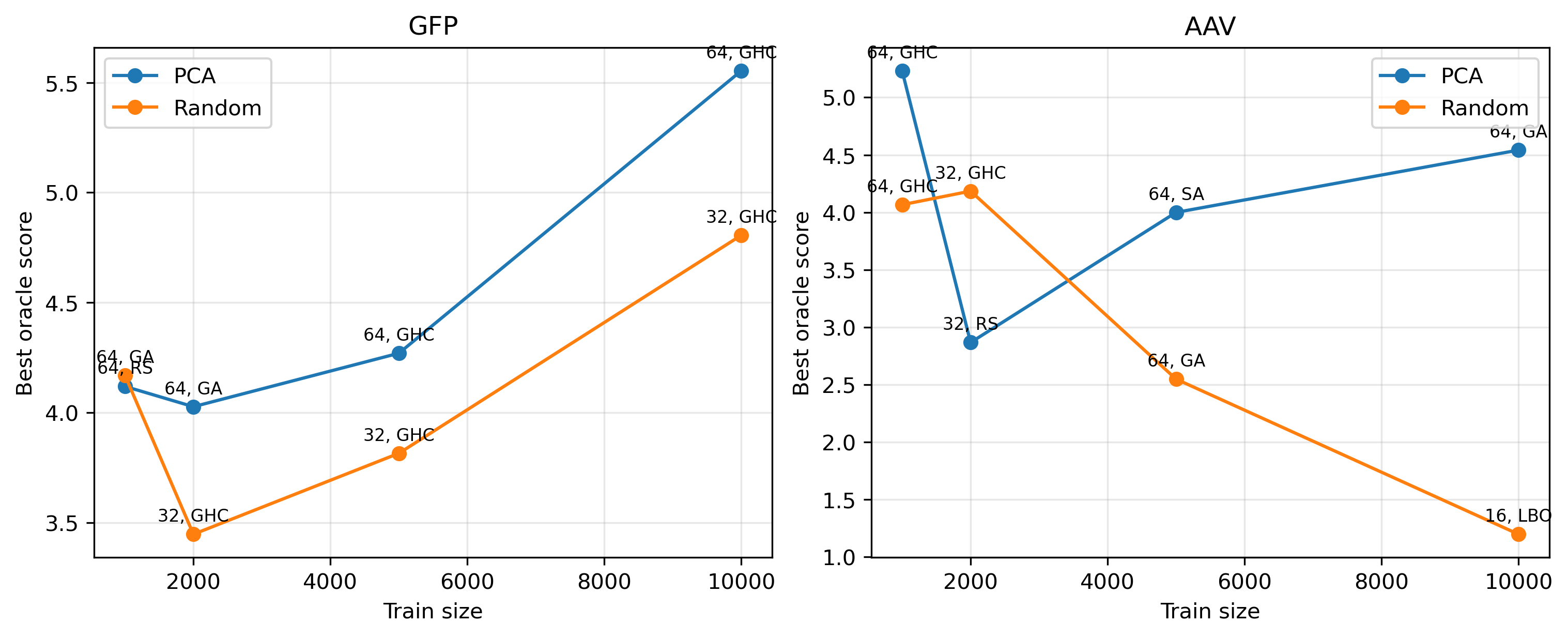}
\caption{End-to-end sequence design performance after optimization, decoding, and oracle scoring. 
For each dataset and training size, we plot the best-performing PCA-based and random-projection configuration. 
Each point is annotated with the corresponding latent dimension and optimizer. 
Across most settings, PCA-based binary latent representations achieve the strongest sequence-level performance, particularly at moderate and larger data regimes. }
\label{fig:section56_main}
\end{figure*}

\subsection{Data-Scarcity and Scaling Analysis}
\label{subsec:data_scarcity}

We next analyze how the performance of \modelname evolves as the amount of available fitness data increases. This experiment is particularly important because one of the primary motivations for using compact binary latent representations is to enable robust optimization under limited supervision.

Table~\ref{tab:scaling_summary} summarizes the best-performing oracle, latent representation, and design method across data regimes on both GFP and AAV. Several consistent patterns emerge.

\paragraph{Oracle behavior across data regimes.}
Across the evaluated datasets and training sizes, Gaussian process regression generally provides the strongest sequence-level oracle, achieving the highest test Spearman correlation in seven of the eight settings; Ridge Regression performs best for AAV with 1,000 samples. This highlights the strong performance of Gaussian processes across most regimes while also showing that no evaluator uniformly dominates every setting.

\paragraph{Latent representation.}
Across all datasets and data regimes, PCA-based binary latent representations consistently provide the strongest performance. This observation aligns with the findings in Section~\ref{subsec:repr_decoder}, where PCA was shown to maintain high entropy, active latent dimensions, and strong decoding performance, while AE/VAE representations collapse after binarization.

\paragraph{Optimization methods.}
The best-performing design method varies with both dataset and data regime. In low-data settings, simpler methods such as Random Search (RS) and Greedy Hill Climbing (GHC) are often competitive, reflecting the relatively limited structure of the learned landscape. As the dataset size increases, more sophisticated combinatorial methods such as Genetic Algorithm (GA) and Simulated Annealing (SA) become increasingly effective, suggesting that the learned QUBO landscape becomes more informative and structured.

\paragraph{Scaling trends.}
We observe a clear transition from less stable behavior in the low-data regime (1000–2000 samples) to more consistent and high-performing optimization in the moderate-data regime (5000–10000 samples). In particular, PCA-based latent representations combined with higher latent dimensionality enable improved end-to-end performance, indicating that structured latent spaces can effectively leverage additional data without sacrificing decodability or optimization stability.

\paragraph{Summary.}
Overall, these results demonstrate that \modelname provides a robust framework for protein design across data regimes. The combination of a strong Gaussian process oracle, structured PCA-based latent representations, and combinatorial optimization enables effective identification of high-fitness sequences, particularly in data-scarce settings, while continuing to improve as more data becomes available.

\begin{table*}[t]
\centering
\caption{Scaling summary across data regimes. We report the best-performing oracle, latent representation, and design method as a function of available training data. \revised{The best latent representation and design method are selected according to the top-10 mean, which reflects robust candidate-set performance rather than a single extreme prediction.}}
\label{tab:scaling_summary}
\resizebox{\textwidth}{!}{%
\begin{tabular}{l c l l l}
\toprule
Dataset & Train size & Best oracle & Best latent representation & Best design method \\
\midrule
\multirow{4}{*}{GFP} & 1000 & Gaussian Process & PCA & Random Search (RS) \\
 & 2000 & Gaussian Process & PCA & Genetic Algorithm (GA) \\
 & 5000 & Gaussian Process & PCA & Greedy Hill Climb (GHC) \\
 & 10000 & Gaussian Process & PCA & Greedy Hill Climb (GHC) \\
\midrule
\multirow{4}{*}{AAV} & 1000 & Gaussian Process & PCA & Greedy Hill Climb (GHC) \\
 & 2000 & Gaussian Process & PCA & Random Search (RS) \\
 & 5000 & Gaussian Process & PCA & Simulated Annealing (SA) \\
 & 10000 & Gaussian Process & PCA & Genetic Algorithm (GA) \\
\bottomrule
\end{tabular}
}
\end{table*}

\subsection{Implementation Details and Reporting Protocol}
\label{subsec:implementation_details}

Unless otherwise specified, all methods use the same upstream ESM embedding extractor in order to isolate the effect of latent representation and optimization strategy. 
\revised{
Model-family selection is performed using validation performance, while the model hyperparameters are fixed as specified in Appendix \ref{sec:additional_experiments}.
}
For design methods, we use a fixed candidate budget and report both the best returned sequence and the average performance of the top-\(K\) returned candidates.
For stochastic optimization, reported summary statistics are aggregated over five independent runs using seeds 0, 1, 2, 3, and 4. Model-fitting components use the fixed seeds specified in Appendix~\ref{sec:additional_experiments} unless otherwise noted.
For this work, we publicly release code, exact hyperparameter settings, and scripts for reproducing all tables and figures at: 
\begin{center}
    \url{https://github.com/HySonLab/Q-BIOLAT-Extended}
\end{center}
\section{Conclusions} \label{sec:conclusion}

In this work, we introduced \modelname, a framework for modeling and optimizing protein fitness landscapes in compact binary latent spaces. By transforming protein sequence representations into binary codes and learning a quadratic unconstrained binary optimization (QUBO) surrogate, \modelname enables direct application of combinatorial optimization methods to protein design.

Beyond its algorithmic formulation, our central contribution is a representation-centric perspective on protein fitness modeling. We show both theoretically and empirically that predictive accuracy alone is insufficient for effective optimization. Instead, the geometry of the latent space---captured through its induced QUBO interactions---plays a decisive role in determining search behavior, landscape smoothness, and the ability to identify high-fitness sequences.

Our experiments demonstrate three key findings. First, learned autoencoder-based representations collapse under binarization and fail to support optimization, despite achieving low reconstruction error. Second, simple structured representations such as PCA consistently produce more informative and decodable latent spaces, leading to better optimization performance. Third, combinatorial optimization methods such as simulated annealing, genetic algorithms, and greedy hill climbing are highly effective when applied to these structured binary latent landscapes, especially when combined with a 
\revised{sequence-level evaluator trained on experimentally measured fitness.}

We further show that \modelname performs robustly across data regimes, including data-scarce settings where traditional methods often struggle. As more data becomes available, both oracle quality and latent structure improve, leading to 
\revised{
increasingly strong model-based end-to-end candidate prioritization.
}

Overall, \modelname bridges protein language modeling, discrete representation learning, and combinatorial optimization, providing a principled framework for optimization-aware protein fitness modeling. By explicitly modeling the latent fitness landscape as a QUBO problem, our approach naturally connects modern machine learning with classical and quantum optimization paradigms. This opens several promising directions for future work, including integration with quantum annealing hardware, development of optimization-aware representation learning methods, and extension to multi-objective and constrained protein design.

\section*{Abbreviations}
AAV, adeno-associated virus; AE, autoencoder; DMS, deep mutational scanning; ESM, Evolutionary Scale Modeling; GA, genetic algorithm; GFP, green fluorescent protein; GHC, greedy hill climbing; LBO, latent Bayesian optimization; PCA, principal component analysis; PLM, protein language model; QUBO, quadratic unconstrained binary optimization; SA, simulated annealing; VAE, variational autoencoder.

\section*{Declarations}

\subsection*{Ethics approval and consent to participate}
Not applicable. This computational study uses publicly available protein fitness datasets and does not involve human participants, human data, or animal experiments.

\subsection*{Consent for publication}
Not applicable.

\subsection*{Availability of data and materials}
The implementation and data resources used for the analyses are available at \url{https://github.com/HySonLab/Q-BIOLAT-Extended}. The experiments use experimentally measured GFP and AAV fitness landscapes from the ProteinGym benchmark, as described in the manuscript.

\subsection*{Competing interests}
The author declares no competing interests.

\subsection*{Funding}
This research received no specific funding.

\subsection*{Authors' contributions}
T.-S.H. is the sole author of this manuscript and is responsible for the work presented herein.

\subsection*{Acknowledgements}
Not applicable.

\newpage

\appendix
\section{Theoretical Insights into Binary Latent Fitness Landscapes} \label{sec:theory}

In this section, we provide a theoretical perspective on how binary latent representations induce structured optimization landscapes, and how their geometry affects combinatorial search.

\subsection{QUBO as an Energy Landscape}

Recall that the QUBO surrogate defines the predicted fitness of a binary latent code $x \in \{0,1\}^m$ as:
\begin{equation}
\hat{f}(x) = h^\top x + \frac{1}{2} x^\top J x,
\end{equation}
where $h \in \mathbb{R}^m$ and $J \in \mathbb{R}^{m \times m}$ is a symmetric interaction matrix. This formulation can be equivalently interpreted as an energy-based model:
\begin{equation}
E(x) = -\hat{f}(x),
\end{equation}
where optimization corresponds to finding low-energy configurations.

Under this view, the binary latent space $\{0,1\}^m$ forms a discrete configuration space (a Boolean hypercube), and the QUBO defines an energy landscape over its vertices. The structure of this landscape is fully determined by the parameters $(h, J)$, which in turn depend on the choice of latent representation.

\subsection{Local Geometry and Bit-Flip Dynamics}

A fundamental operation in combinatorial optimization is a single-bit flip. Let $x^{(k)}$ denote the vector obtained by flipping the $k$-th bit of $x$. The change in predicted fitness is:
\begin{equation}
\Delta_k(x) = \hat{f}(x^{(k)}) - \hat{f}(x).
\end{equation}

For the QUBO model, this quantity can be expressed explicitly as:
\begin{equation}
\Delta_k(x) = (1 - 2x_k)\left(h_k + \sum_{\ell \neq k} J_{k\ell} x_\ell \right).
\end{equation}

This expression shows that local search behavior is governed by the interaction structure of $J$. In particular:
\begin{itemize}
    \item The term $h_k$ captures the intrinsic contribution of the $k$-th latent dimension.
    \item The sum $\sum_{\ell} J_{k\ell} x_\ell$ captures how other latent variables influence the effect of flipping bit $k$.
\end{itemize}

Thus, optimization trajectories depend not only on the current state $x$, but also on the global interaction structure encoded in $J$.

\subsection{Landscape Smoothness and Ruggedness}

The difficulty of optimization is closely related to the smoothness of the energy landscape. We characterize local variability using the variance of $\Delta_k(x)$ over the latent space:
\begin{equation}
\mathrm{Var}[\Delta_k(x)] \approx \sum_{\ell} J_{k\ell}^2.
\end{equation}

This suggests that the norm of the interaction coefficients controls the ruggedness of the landscape:
\begin{itemize}
    \item Small $\|J_{k\cdot}\|$ leads to smoother landscapes with more predictable local improvements.
    \item Large $\|J_{k\cdot}\|$ leads to highly variable local changes and a proliferation of local optima.
\end{itemize}

Consequently, the structure of $J$ plays a central role in determining optimization difficulty.

\subsection{Spectral Structure of the Interaction Matrix}

Further insight can be obtained by analyzing the spectrum of the interaction matrix $J$. Let $\{\lambda_i\}$ denote its eigenvalues. Then:
\begin{itemize}
    \item Low-rank or structured $J$ (few dominant eigenvalues) corresponds to landscapes with coherent global directions, facilitating optimization.
    \item Random or full-rank $J$ with dispersed eigenvalues corresponds to more irregular, spin-glass-like landscapes with many local minima.
\end{itemize}

Since $J$ is learned from latent representations, its spectral properties are directly influenced by the choice of representation (e.g., PCA vs random projection vs learned encoders).

\subsection{Representation--Optimization Decoupling}

A key implication of this analysis is that predictive performance alone does not determine optimization behavior. Two representations may yield similar surrogate accuracy (e.g., similar Spearman correlation), yet induce different interaction matrices $J$ with distinct geometric properties.

\begin{quote}
\textbf{Observation:} Representations with similar predictive accuracy can induce fundamentally different optimization landscapes.
\end{quote}

This phenomenon arises because predictive metrics evaluate pointwise accuracy, whereas optimization depends on the global structure of the energy landscape. As a result, representation learning for protein design should account not only for prediction quality, but also for the induced landscape geometry.

\subsection{Implications for Latent Representation Design}

The above analysis suggests several principles for designing effective latent representations:
\begin{itemize}
    \item Structured representations (e.g., PCA or learned encoders) tend to produce smoother and more optimization-friendly landscapes.
    \item Overly high-dimensional latent spaces can increase model expressivity but also introduce ruggedness, making optimization more difficult.
    \item Learned representations can potentially balance expressivity and smoothness by aligning latent dimensions with fitness-relevant directions.
\end{itemize}

These insights provide a theoretical foundation for understanding the empirical observations in Section~\ref{sec:experiments} and motivate the development of optimization-aware representation learning methods.

Here, we formalize how the learned binary representation determines the geometry, smoothness, and identifiability of the induced QUBO landscape. Our goal is not to solve the global QUBO exactly, but to characterize the quantities that govern local search, global variation, and the stability of the learned surrogate.

\subsection{Energy Landscape and Local Fields}

Recall that the QUBO surrogate defines the predicted fitness of a binary latent code $x \in \{0,1\}^m$ as:
\begin{equation}
\hat{f}(x) = h^\top x + \frac{1}{2} x^\top J x,
\end{equation}
where $h \in \mathbb{R}^m$ and $J \in \mathbb{R}^{m \times m}$ is a symmetric interaction matrix with zero diagonal. Equivalently, we may define the energy:
\begin{equation}
E(x) = -\hat{f}(x),
\end{equation}
so that maximizing predicted fitness is equivalent to minimizing energy over the Boolean hypercube.

A useful quantity for analyzing this landscape is the \emph{local field}:
\begin{equation}
g_k(x) = h_k + \sum_{\ell \neq k} J_{k\ell} x_\ell,
\end{equation}
which measures the effective force acting on the $k$-th bit given the remaining coordinates.

\paragraph{Proposition (Bit-flip gain and local optimality).}
Let $x^{(k)}$ denote the binary code obtained by flipping the $k$-th bit of $x$. Then the change in predicted fitness induced by a single-bit flip is:
\begin{equation}
\Delta_k(x) = \hat{f}(x^{(k)}) - \hat{f}(x)
= (1 - 2x_k)\, g_k(x).
\end{equation}
Consequently, a binary code $x^\star$ is a local maximizer with respect to single-bit flips if and only if:
\begin{equation}
\Delta_k(x^\star) \leq 0 \qquad \text{for all } k,
\end{equation}
or equivalently,
\begin{equation}
x_k^\star =
\begin{cases}
1, & g_k(x^\star) \geq 0,\\
0, & g_k(x^\star) \leq 0,
\end{cases}
\end{equation}
with either value allowed when $g_k(x^\star)=0$.

\paragraph{Proof.}
Write $x^{(k)} = x + (1-2x_k)e_k$, where $e_k$ is the $k$-th standard basis vector. Substituting into the QUBO objective and using the symmetry of $J$ together with $J_{kk}=0$ gives:
\[
\hat{f}(x^{(k)}) - \hat{f}(x)
= (1-2x_k)\left(h_k + \sum_{\ell \neq k} J_{k\ell} x_\ell\right).
\]
The local-optimality condition follows immediately by requiring that no single-bit flip improves the objective.
\hfill $\square$

This proposition shows that local search in \modelname is governed by the collection of local fields $\{g_k(x)\}_{k=1}^m$. In particular, the interaction matrix $J$ determines how strongly each latent dimension depends on the current configuration of the others.

\subsection{Hamming Smoothness and Spectral Control}

We next quantify how much the QUBO objective can change between two latent codes.

\paragraph{Proposition (Hamming--Lipschitz continuity).}
For any $x,y \in \{0,1\}^m$,
\begin{equation}
|\hat{f}(x) - \hat{f}(y)|
\leq \left( \|h\|_\infty + \|J\|_\infty \right)\, d_H(x,y),
\end{equation}
where $d_H(x,y)$ is the Hamming distance and $\|J\|_\infty = \max_k \sum_{\ell=1}^m |J_{k\ell}|$ is the matrix infinity norm.

\paragraph{Proof.}
Let $\delta = x-y$. Then $\|\delta\|_1 = d_H(x,y)$ and
\begin{equation}
\hat{f}(x)-\hat{f}(y)
= h^\top \delta + \frac{1}{2}\delta^\top J (x+y).
\end{equation}
Therefore,
\[
|\hat{f}(x)-\hat{f}(y)|
\leq \|h\|_\infty \|\delta\|_1
+ \frac{1}{2}\|\delta\|_1 \|J\|_\infty \|x+y\|_\infty.
\]
Since $x,y \in \{0,1\}^m$, we have $\|x+y\|_\infty \leq 2$, which yields the result.
\hfill $\square$

This result shows that the quantity
\begin{equation}
L_H := \|h\|_\infty + \|J\|_\infty
\end{equation}
acts as a global smoothness constant of the landscape in Hamming space.

\paragraph{Corollary (Spectral control of landscape variation).}
For any $x,y \in \{0,1\}^m$,
\begin{equation}
|\hat{f}(x) - \hat{f}(y)|
\leq \sqrt{d_H(x,y)}
\left( \|h\|_2 + \sqrt{m}\,\|J\|_2 \right),
\end{equation}
where $\|J\|_2$ denotes the spectral norm of $J$.

\paragraph{Proof.}
Using the same decomposition with $\delta=x-y$,
\[
|\hat{f}(x)-\hat{f}(y)|
\leq \|h\|_2 \|\delta\|_2
+ \frac{1}{2}\|\delta\|_2 \|J\|_2 \|x+y\|_2.
\]
Now $\|\delta\|_2 \leq \sqrt{d_H(x,y)}$ and $\|x+y\|_2 \leq 2\sqrt{m}$, which implies the bound.
\hfill $\square$

This corollary directly motivates the spectral norm of $J$ as a global diagnostic of landscape roughness. In particular, smaller $\|J\|_2$ implies smaller worst-case variation across the hypercube.

\subsection{Ruggedness via Bit-Flip Variability}

To study local ruggedness, we analyze the variability of single-bit moves over the hypercube.

\paragraph{Proposition (Exact second moment of bit-flip gains).}
Assume $x$ is drawn uniformly from $\{0,1\}^m$. Then for each bit $k$,
\begin{equation}
\mathbb{E}[\Delta_k(x)] = 0,
\end{equation}
and
\begin{equation}
\mathbb{E}[\Delta_k(x)^2]
=
\left(h_k + \frac{1}{2}\sum_{\ell \neq k} J_{k\ell}\right)^2
+ \frac{1}{4}\sum_{\ell \neq k} J_{k\ell}^2.
\end{equation}

\paragraph{Proof.}
Using $\Delta_k(x) = (1-2x_k)g_k(x)$, note that $(1-2x_k)$ is independent of $g_k(x)$ and has mean zero and squared value one under the uniform measure. Hence,
\[
\mathbb{E}[\Delta_k(x)] = \mathbb{E}[1-2x_k]\;\mathbb{E}[g_k(x)] = 0,
\]
and,
\[
\mathbb{E}[\Delta_k(x)^2] = \mathbb{E}[g_k(x)^2].
\]
Since the coordinates $\{x_\ell\}_{\ell\neq k}$ are independent Bernoulli$(1/2)$ variables,
\[
\mathbb{E}[g_k(x)] = h_k + \frac{1}{2}\sum_{\ell \neq k} J_{k\ell},
\qquad
\mathrm{Var}(g_k(x)) = \frac{1}{4}\sum_{\ell \neq k} J_{k\ell}^2.
\]
The result follows from $\mathbb{E}[g_k(x)^2] = \mathrm{Var}(g_k(x)) + \mathbb{E}[g_k(x)]^2$.
\hfill $\square$

This proposition yields an exact measure of local ruggedness. We define the average ruggedness:
\begin{equation}
\mathcal{R}(J,h)
=
\frac{1}{m}\sum_{k=1}^m \mathbb{E}[\Delta_k(x)^2].
\end{equation}
When the latent bits are approximately balanced and the local fields are centered, the dominant contribution is:
\begin{equation}
\mathcal{R}(J,h) \approx \frac{1}{4m}\|J\|_F^2,
\end{equation}
showing that interaction energy directly drives local variability. Thus, both $\|J\|_F$ and $\|J\|_2$ are theoretically meaningful diagnostics of ruggedness.

\subsection{Low-Rank Structure and Effective Optimization Dimension}

The spectral structure of $J$ determines whether the landscape is effectively governed by a small number of collective directions. Let
\begin{equation}
J = \sum_{i=1}^m \lambda_i u_i u_i^\top
\end{equation}
be the spectral decomposition of $J$, with eigenvalues ordered by decreasing magnitude:
\[
|\lambda_1| \geq |\lambda_2| \geq \cdots \geq |\lambda_m|.
\]
For a given rank $r$, we define the truncated interaction matrix:
\begin{equation}
J_r = \sum_{i=1}^r \lambda_i u_i u_i^\top,
\end{equation}
and the corresponding truncated objective:
\begin{equation}
\hat{f}_r(x) = h^\top x + \frac{1}{2}x^\top J_r x.
\end{equation}

\paragraph{Proposition (Pointwise low-rank approximation).}
For every $x \in \{0,1\}^m$,
\begin{equation}
|\hat{f}(x) - \hat{f}_r(x)|
\leq \frac{m}{2}\,\|J-J_r\|_2.
\end{equation}

\paragraph{Proof.}
We have:
\[
\hat{f}(x)-\hat{f}_r(x)
=
\frac{1}{2}x^\top (J-J_r)x.
\]
Therefore,
\[
|\hat{f}(x)-\hat{f}_r(x)|
\leq \frac{1}{2}\|J-J_r\|_2 \|x\|_2^2.
\]
Since $x \in \{0,1\}^m$, we have $\|x\|_2^2 \leq m$, which proves the result.
\hfill $\square$

\paragraph{Corollary (Optimization gap under low-rank truncation).}
Let $x^\star \in \arg\max_x \hat{f}(x)$ and $x_r^\star \in \arg\max_x \hat{f}_r(x)$. Then,
\begin{equation}
\hat{f}(x^\star) - \hat{f}(x_r^\star)
\leq m\,\|J-J_r\|_2.
\end{equation}

\paragraph{Proof.}
From the previous proposition, $|\hat{f}(x)-\hat{f}_r(x)| \leq \varepsilon$ for all $x$, where $\varepsilon = \frac{m}{2}\|J-J_r\|_2$. Hence,
\[
\hat{f}(x^\star)
\leq \hat{f}_r(x^\star) + \varepsilon
\leq \hat{f}_r(x_r^\star) + \varepsilon
\leq \hat{f}(x_r^\star) + 2\varepsilon.
\]
Substituting the value of $\varepsilon$ yields the claim.
\hfill $\square$

This corollary provides a concrete guarantee: if the spectral tail of $J$ is small, then the optimization landscape is effectively low-dimensional, and the dominant eigenmodes of $J$ capture most of the optimization-relevant structure.

A useful scalar summary of this concentration is the \emph{effective rank}:
\begin{equation}
r_{\mathrm{eff}}(J) = \frac{\|J\|_F^2}{\|J\|_2^2},
\end{equation}
which is small when interaction energy is concentrated in a few dominant modes.

\subsection{Identifiability of the Learned QUBO Landscape}

Let $\phi(x) \in \mathbb{R}^p$ denote the QUBO feature map containing all linear and pairwise terms, where
\begin{equation}
p = m + \frac{m(m-1)}{2}.
\end{equation}
Let $\Phi(X) \in \mathbb{R}^{N \times p}$ be the design matrix formed from the observed latent codes in the training set.

\paragraph{Proposition (Non-identifiability under rank deficiency).}
If $\mathrm{rank}(\Phi(X)) < p$, then the QUBO parameters are not uniquely identifiable from the observed data. In particular, there exists a nonzero vector $u \in \mathbb{R}^p$ such that:
\begin{equation}
\Phi(X)u = 0.
\end{equation}
Therefore, for any parameter vector $w$, both $w$ and $w+u$ induce identical predictions on all observed training codes, but can differ on unseen binary codes.

\paragraph{Proof.}
If $\mathrm{rank}(\Phi(X)) < p$, then the null space of $\Phi(X)$ is nontrivial, so there exists $u \neq 0$ with $\Phi(X)u=0$. Hence,
\[
\Phi(X)(w+u) = \Phi(X)w,
\]
so the two parameterizations agree on all observed codes. For an unseen code $x$ with $\phi(x)^\top u \neq 0$, the predicted fitness differs:
\[
\phi(x)^\top(w+u) \neq \phi(x)^\top w.
\]
\hfill $\square$

\paragraph{Corollary (Dimension--sample trade-off).}
Since $p = \mathcal{O}(m^2)$, increasing the latent dimension rapidly enlarges the number of QUBO parameters. When the number of observed sequences is not sufficiently large relative to $p$, the landscape away from the observed codes is only weakly constrained.

This result provides a formal explanation for the trade-off between latent dimensionality, surrogate generalization, and optimization stability. Ridge regularization selects one stable solution among many possibilities, but the induced landscape remains strongly dependent on the representation through the feature map $\Phi(X)$.

\subsection{Why Prediction and Optimization Can Decouple}

The previous results formalize why predictive accuracy alone does not determine optimization behavior.

First, predictive metrics such as Spearman correlation evaluate surrogate quality on a finite observed set, whereas optimization depends on the geometry of the entire hypercube. Second, our propositions show that when the QUBO feature map is rank-deficient or weakly constrained, multiple interaction structures can be consistent with essentially the same observed fit while differing away from the observed manifold. 

Together, these observations imply that two latent representations can achieve similar predictive accuracy while inducing substantially different optimization landscapes. This representation--optimization decoupling is not an anomaly, but rather a structural consequence of learning a surrogate over a large discrete latent space.

\subsection{Practical Consequences}

The theory above yields several experimentally testable predictions. Landscapes with smaller $\|J\|_2$, faster spectral decay, and smaller ruggedness $\mathcal{R}(J,h)$ should be easier to optimize. Likewise, increasing latent dimensionality without sufficient data should increase ambiguity in the learned landscape and make optimization less stable.

Further empirical analyses of these theoretical predictions are left for future work.
\section{Additional Experimental Details} \label{sec:additional_experiments}

\subsection{Protein Sequence Embeddings (ESM)}

We represent protein sequences using pretrained protein language models from the ESM family. In particular, we use the \texttt{facebook/esm2\_t6\_8M\_UR50D} checkpoint from the Hugging Face Transformers library.

Given a protein sequence, we first convert it to uppercase and remove non-alphabetic characters. The cleaned sequence is then tokenized using the corresponding ESM tokenizer. We apply padding and truncation with a maximum sequence length of 1024.

The tokenized sequence is passed through the pretrained transformer model in evaluation mode without gradient computation. The model produces contextualized residue-level representations, from which we obtain a fixed-length sequence embedding by mean pooling over the sequence dimension using the attention mask:
\begin{equation}
e = \frac{\sum_{i=1}^{L} m_i H_i}{\sum_{i=1}^{L} m_i},
\end{equation}
where $H_i$ denotes the hidden representation of the $i$-th residue and $m_i$ is the corresponding attention mask.

Embeddings are computed in mini-batches of size 8 and stored as dense vectors in \texttt{float32} format. These embeddings are used as inputs to downstream components, including the external fitness oracle and latent representation learning.

\subsection{External Sequence-Level Fitness Oracle}

For sequence-level fitness evaluation, we trained an external oracle on dense protein sequence embeddings. Each dataset was stored as a NumPy archive containing (i) dense embeddings, (ii) scalar fitness labels, and (iii) the corresponding protein sequences. The embeddings were represented in \texttt{float32}, while fitness targets were stored in \texttt{float64}.

\paragraph{Data splitting.}
We partitioned each dataset into training, validation, and test sets using a two-stage random split. First, we split the data into a training + validation set and a held-out test set with a test ratio of $0.2$. Then, we further split the training + validation set into training and validation subsets using a validation ratio of $0.1$ relative to the full dataset. This results in an effective split of $70\%/10\%/20\%$ for training, validation, and test sets, respectively. All splits were performed with a fixed random seed of 42 for reproducibility.

\paragraph{Evaluation metrics.}
We evaluated model performance using four standard regression metrics: Spearman correlation, Pearson correlation, root mean squared error (RMSE), and mean absolute error (MAE). Metrics were reported on the training, validation, and test sets.

\paragraph{Models.}
We considered three regression models on top of the same embedding representation: Ridge Regression, XGBoost, and Gaussian Process Regression.

\paragraph{Ridge Regression.}
We used a pipeline consisting of feature standardization via \texttt{StandardScaler} followed by Ridge regression. The regularization strength was set to $\alpha = 1.0$, and the \texttt{svd} solver was used.

\paragraph{XGBoost.}
We used the \texttt{XGBRegressor} implementation with the following hyperparameters: number of trees $=300$, maximum tree depth $=6$, learning rate $=0.05$, subsampling ratio $=0.9$, column subsampling ratio $=0.9$, $\ell_1$ regularization coefficient $=0.0$, and $\ell_2$ regularization coefficient $=1.0$. The objective function was squared error regression. 
Training was performed using 4 CPU threads, and the random seed was fixed to 42.

\paragraph{Gaussian Process Regression.}
We used a pipeline consisting of feature standardization followed by Gaussian Process Regression (GPR). The kernel was defined as:
\begin{equation}
k(x, x') = C \cdot \mathrm{RBF}(x, x') + \mathrm{WhiteKernel},
\end{equation}
where the constant kernel $C$ was initialized to $1.0$ with bounds $[10^{-3}, 10^{3}]$, the RBF kernel had an initial length scale of $1.0$ with bounds $[10^{-3}, 10^{3}]$, and the white noise level was initialized to $10^{-3}$ with bounds $[10^{-6}, 10^{1}]$. The GPR model used \texttt{normalize\_y=True} and performed $2$ restarts of the optimizer.

\paragraph{Training protocol.}
All models were trained on both GFP and AAV datasets under four data regimes with training set sizes of $\{1000, 2000, 5000, 10000\}$. For each dataset and data regime, the same fixed hyperparameters and random seed were used. No additional hyperparameter search was performed. Models were trained on the training split and evaluated on validation and test splits.

\paragraph{Outputs.}
For each trained model, we stored (i) the serialized model parameters, (ii) evaluation metrics in JSON format, and (iii) test-set predictions including sequences, ground-truth fitness values, and predicted fitness values.

\subsection{Latent Representation Learning and Decoder Implementation}

We construct binary latent representations from dense ESM embeddings using four approaches: PCA, random projection, deterministic autoencoder (AE), and variational autoencoder (VAE). For all methods, we evaluate latent dimensions of 8, 16, 32, and 64.

For PCA, we apply principal component analysis to the dense embeddings in each sampled dataset and retain the top components corresponding to the target latent dimension. The PCA transformation and per-dimension median thresholds are computed once from that sampled embedding set before the downstream supervised splits, and the resulting binary representations are precomputed and reused across experiments. This is an unsupervised representation-construction step: fitness labels are not used when fitting PCA or computing the thresholds.

For random projection, we project embeddings using a Gaussian random matrix with variance scaled by the input dimension, followed by the same median-based binarization. This provides a lightweight baseline without learned structure.

For AE and VAE, we use simple fully connected neural networks. The autoencoder consists of a two-layer encoder and decoder with a hidden dimension of 256 and ReLU activations. The encoder maps the input embedding to the latent space, and the decoder reconstructs the original embedding. The VAE uses a similar architecture with a shared encoder backbone and separate linear heads for the mean and log-variance. Both models are trained using mean squared reconstruction loss, with an additional KL divergence term for the VAE weighted by $\beta = 10^{-3}$. Latent representations are binarized by thresholding at zero.

Both AE and VAE are trained using the Adam optimizer with learning rate $10^{-3}$, weight decay $10^{-5}$, batch size 32, and 200 training epochs. All models are trained with a fixed random seed of 42 on CPU.

For decoding, we train a mutation-conditioned neural decoder that maps latent codes back to protein sequences. The decoder consists of two fully connected layers with hidden dimension 256 and ReLU activations, followed by two output heads: one for predicting mutation positions and one for predicting amino acid identities at mutated positions. The mutation prediction head outputs a vector of length equal to the sequence length, while the amino acid head outputs logits over 20 amino acids per position.

The decoder is trained using a combination of binary cross-entropy loss for mutation prediction and cross-entropy loss for amino acid prediction, with a higher weight assigned to the mutation loss to account for class imbalance. Training is performed using Adam with learning rate $10^{-3}$, batch size 32, and 100 epochs. The best model is selected based on validation performance.

All latent models and decoders are trained across both GFP and AAV datasets and all data regimes using consistent hyperparameters.

\subsection{Internal QUBO Surrogate Implementation}

The internal surrogate operates directly on binary latent representations of dimension $m \in \{8,16,32,64\}$ and models the fitness function as a combination of linear and pairwise interaction terms. Given a binary latent vector, the surrogate includes all $m$ linear terms and all $\frac{m(m-1)}{2}$ pairwise interaction terms. These features are constructed explicitly by concatenating the original binary variables with all pairwise products between distinct latent dimensions.

The model parameters are fitted using ridge regression with $\ell_2$ regularization. Specifically, the surrogate is trained by solving a linear system involving the feature matrix and the target fitness values, with a regularization coefficient set to $10^{-3}$. This closed-form solution enables efficient training even when the number of pairwise features grows quadratically with the latent dimension.

After training, the learned parameters are decomposed into a vector of linear coefficients and a symmetric matrix of pairwise interaction coefficients. The diagonal entries of the interaction matrix are set to zero, and symmetry is enforced by assigning equal weights to $(i,j)$ and $(j,i)$. Predictions are computed by combining the linear contributions and pairwise interactions over the binary latent variables.

We evaluate the surrogate using standard regression metrics including RMSE, $R^2$, and Spearman correlation on both training and held-out test sets. The model is trained using a fixed random seed of 42 and a train/test split ratio of 80/20.

For comparison, the implementation also includes a small multilayer perceptron (MLP) baseline trained using mini-batch gradient descent with ReLU activations. The MLP uses two hidden layers with dimensions 64 and 32, learning rate $10^{-3}$, weight decay $10^{-5}$, batch size 64, and 400 training epochs. However, this baseline is used only for reference, while all optimization experiments in the main paper rely on the QUBO surrogate.

The surrogate is trained across all dataset sizes and latent dimensions and serves as the objective function for downstream combinatorial optimization methods.

\subsection{Combinatorial Optimization Methods}

We optimize the learned QUBO surrogate in the binary latent space using five methods: Simulated Annealing (SA), Genetic Algorithm (GA), Random Search (RS), Greedy Hill Climbing (GHC), and a lightweight latent Bayesian optimization (LBO) baseline. All methods operate directly on binary vectors of dimension $m \in \{8,16,32,64\}$.

Simulated Annealing (SA) starts from an initial binary code and iteratively proposes single-bit flips. At each step, a bit index is selected uniformly at random, and the energy change induced by flipping that bit is computed efficiently using the learned QUBO parameters. Moves that improve the objective are always accepted, while worse moves are accepted with a temperature-dependent probability. The temperature follows an exponential decay schedule with initial temperature $T_0=1.0$, minimum temperature $10^{-4}$, and decay factor $0.999$. The algorithm runs for 20{,}000 steps and keeps track of the best solution encountered during the trajectory.

The Genetic Algorithm (GA) maintains a population of candidate binary solutions and evolves them over multiple generations. We use a population size of 64 and run the algorithm for 150 generations. At each generation, a subset of top-performing individuals (elite size 4) is preserved. New individuals are generated using tournament selection with tournament size 3, followed by single-point crossover applied with probability 0.9. Offspring are further perturbed using independent bit-flip mutation with mutation rate 0.02. The population is updated at each generation, and the best solution across all generations is returned.

Random Search (RS) samples binary latent codes independently and uniformly at random. We evaluate 10{,}000 candidate solutions and return the one with the highest objective value.

Greedy Hill Climbing (GHC) starts from an initial binary code and iteratively improves it by exhaustively evaluating all single-bit flips. At each iteration, the algorithm selects the bit flip that yields the largest improvement in the objective and updates the current solution accordingly. The process is repeated until no improving move exists or until a maximum of 100 passes over all bits is reached.

The Latent Bayesian Optimization (LBO) baseline uses a kernel-based uncertainty heuristic over binary latent codes. Given a set of seed solutions, the method samples 5{,}000 candidate binary vectors uniformly at random and evaluates an acquisition function based on an upper confidence bound. The predictive mean is computed using a kernel-weighted average over seed points, where the kernel is defined using an RBF function over Hamming distance with length scale 4.0. The acquisition function is given by the sum of the predictive mean and an exploration term weighted by $\beta = 1.0$. The candidate with the highest acquisition value is selected.

All optimization methods are applied consistently across datasets, latent dimensions, and data regimes using fixed hyperparameters. For the multiseed optimization benchmark, each configuration is evaluated with seeds 0, 1, 2, 3, and 4, and reported means aggregate these five runs.

\subsection{Additional Experimental Results}

We provide detailed end-to-end design results across all dataset sizes and latent dimensions in Figures~\ref{fig:appendix_gfp_methods} and~\ref{fig:appendix_aav_methods}, as well as Tables~\ref{tab:appendix_gfp_1000}--\ref{tab:appendix_aav_10000}. Figure~\ref{fig:appendix_gfp_methods} presents the performance trends on the GFP dataset across different training regimes, while Figure~\ref{fig:appendix_aav_methods} shows the corresponding results for the AAV dataset. These figures illustrate how performance varies with latent dimensionality and optimization method. Tables~\ref{tab:appendix_gfp_1000}--\ref{tab:appendix_gfp_10000} report detailed results for the GFP dataset at training sizes of 1000, 2000, 5000, and 10000 samples, respectively, while Tables~\ref{tab:appendix_aav_1000}--\ref{tab:appendix_aav_10000} present the corresponding results for the AAV dataset across all data regimes. Each table includes both the best achieved score and the average performance of the top-10 candidates for each configuration, enabling a more fine-grained comparison across latent representations, optimization methods, and data regimes. Overall, these additional results further corroborate the trends observed in the main text, including the consistent advantage of PCA-based latent representations and the strong performance of combinatorial optimization methods such as greedy hill climbing, genetic algorithms, and simulated annealing.

\begin{figure*}[t]
\centering
\begin{subfigure}[t]{0.48\textwidth}
\centering
\includegraphics[width=\linewidth]{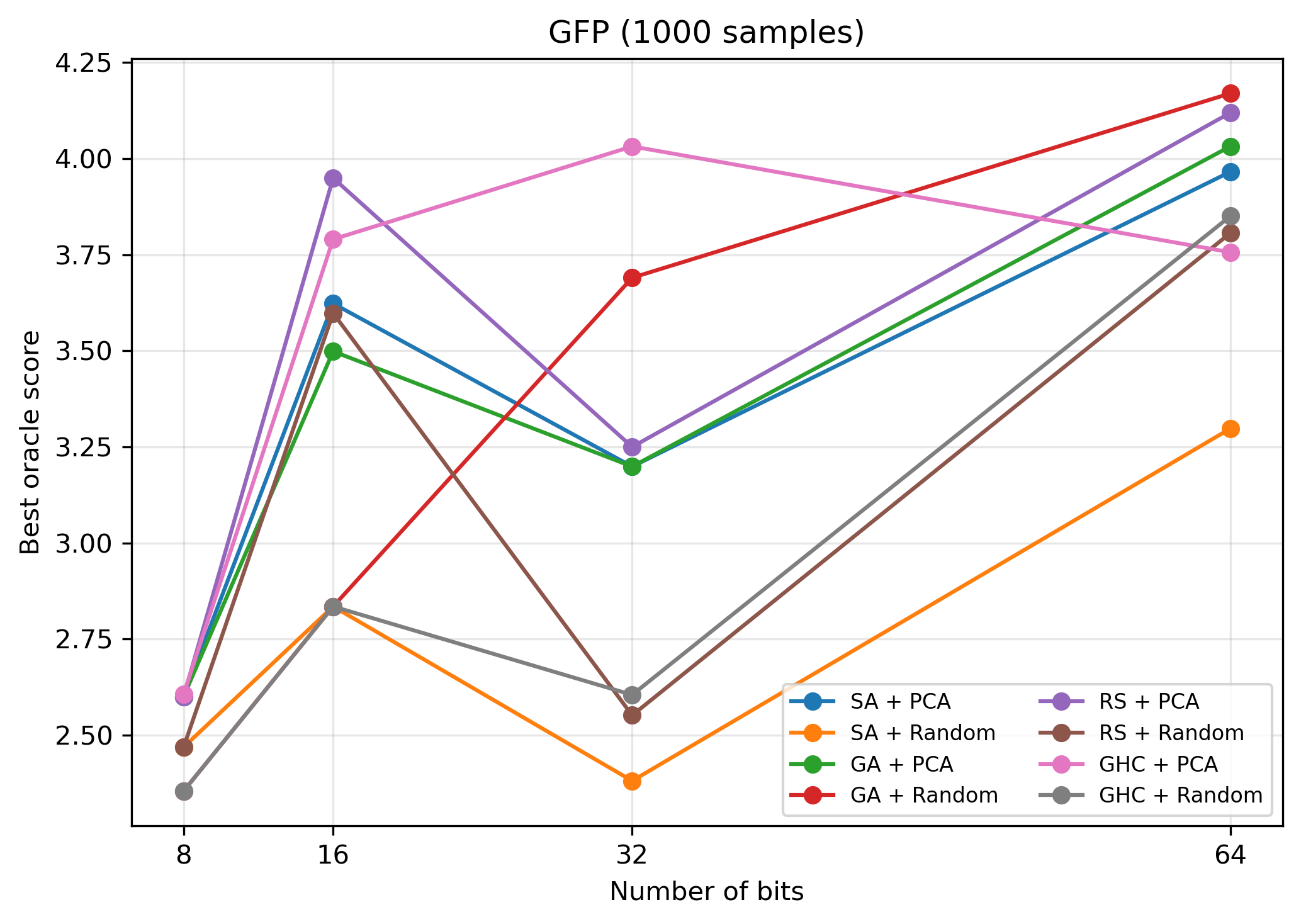}
\caption{GFP with 1000 training samples.}
\end{subfigure}
\hfill
\begin{subfigure}[t]{0.48\textwidth}
\centering
\includegraphics[width=\linewidth]{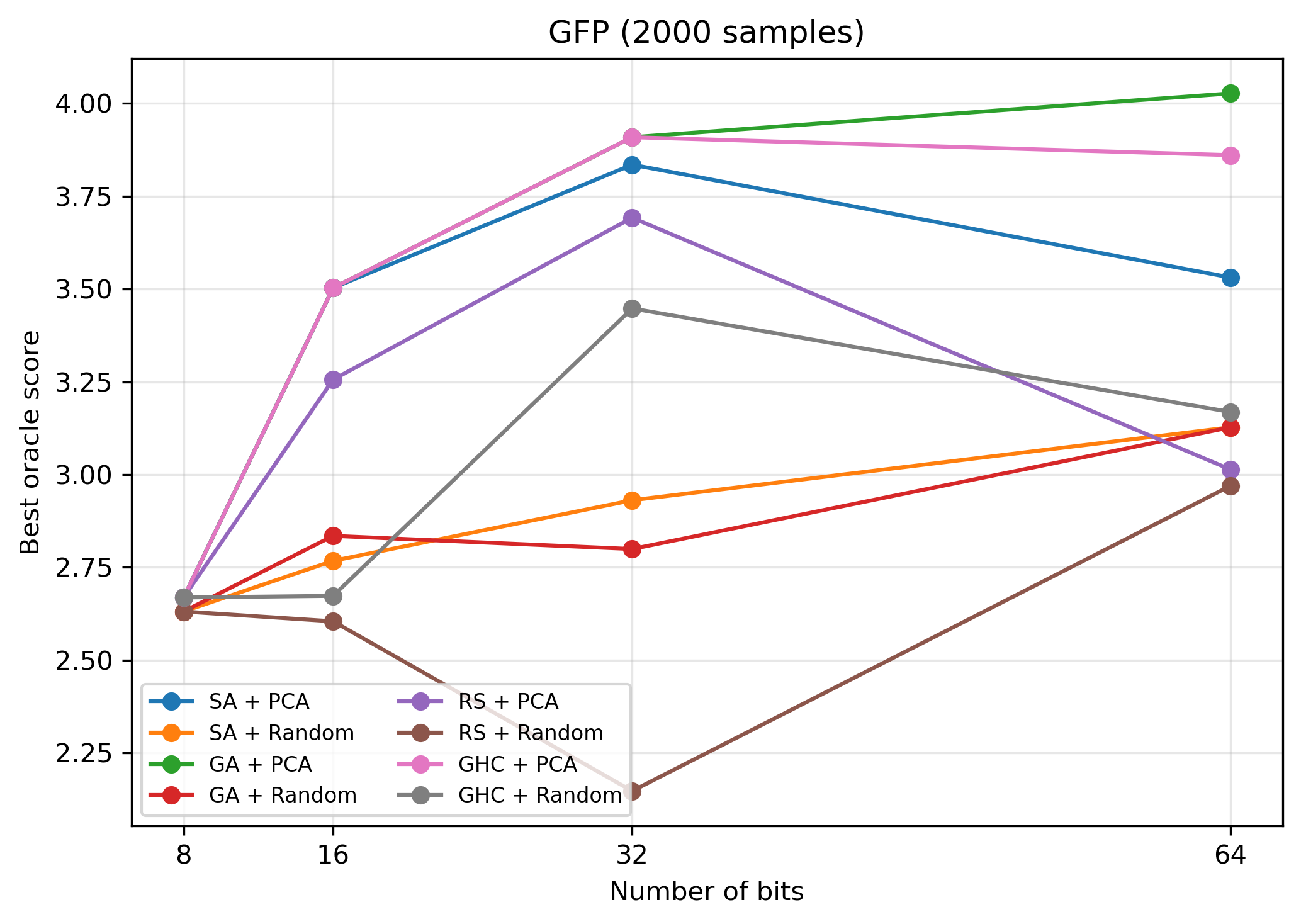}
\caption{GFP with 2000 training samples.}
\end{subfigure}
\vspace{0.5em}
\begin{subfigure}[t]{0.48\textwidth}
\centering
\includegraphics[width=\linewidth]{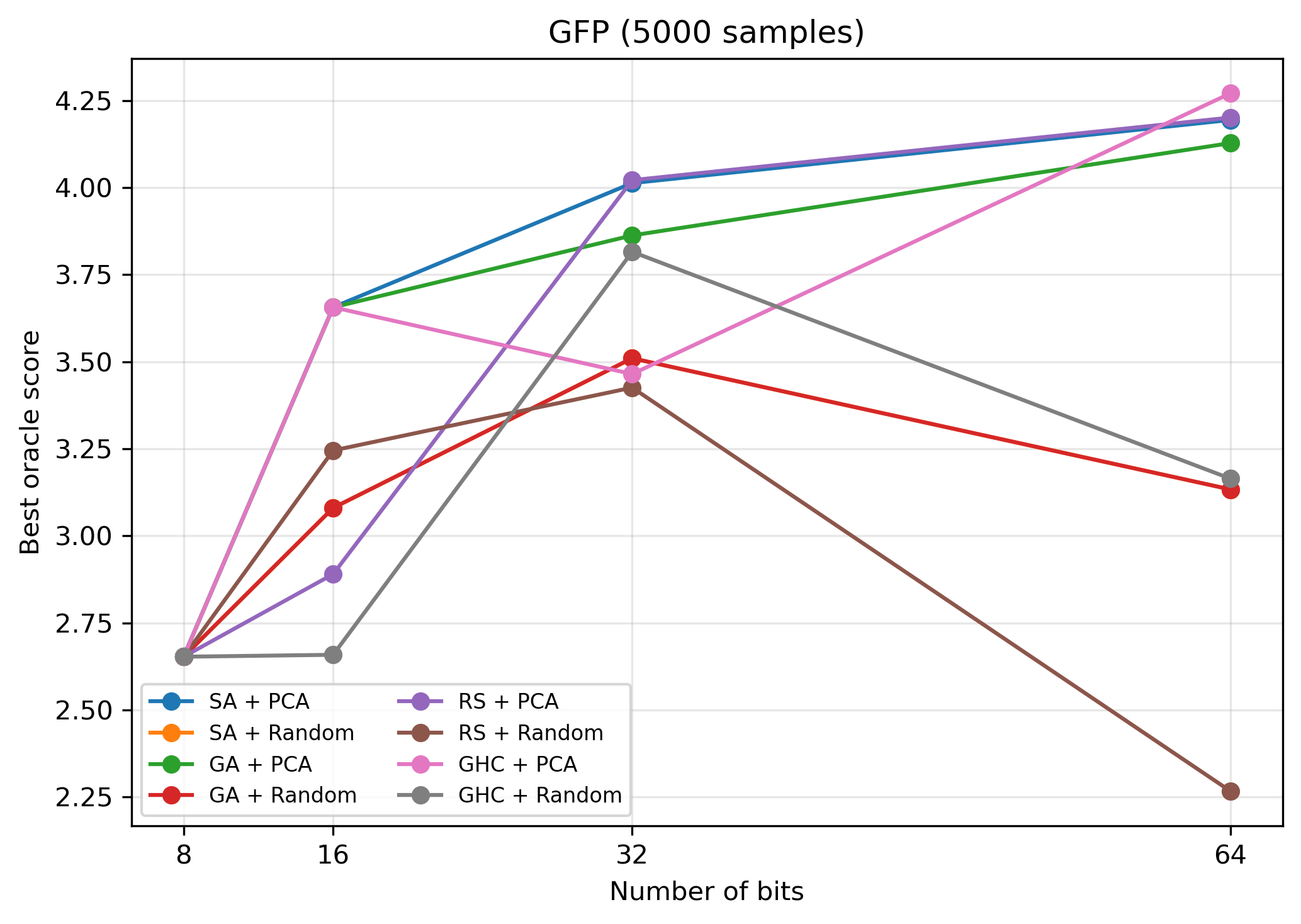}
\caption{GFP with 5000 training samples.}
\end{subfigure}
\hfill
\begin{subfigure}[t]{0.48\textwidth}
\centering
\includegraphics[width=\linewidth]{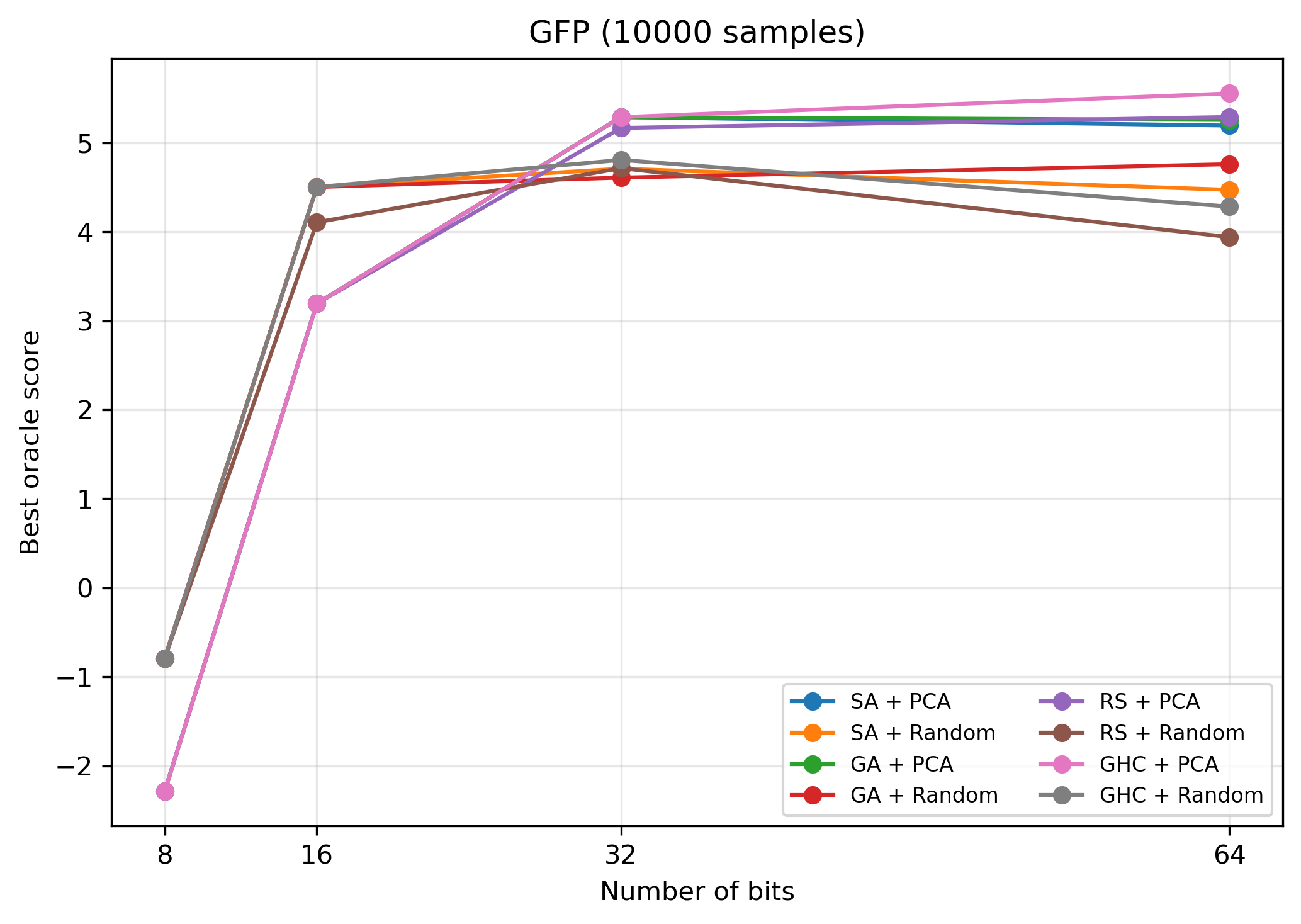}
\caption{GFP with 10000 training samples.}
\end{subfigure}
\vspace{0.5em}
\caption{End-to-end design performance on GFP across different training sizes. The x-axis shows the number of bits and the y-axis shows the best oracle score. Each line corresponds to an optimizer-representation pair, including simulated annealing (SA), genetic algorithm (GA), random search (RS), and greedy hill climbing (GHC), combined with either PCA-based or random-projection binary latent representations.}
\label{fig:appendix_gfp_methods}
\end{figure*}

\begin{figure*}[t]
\centering
\begin{subfigure}[t]{0.48\textwidth}
\centering
\includegraphics[width=\linewidth]{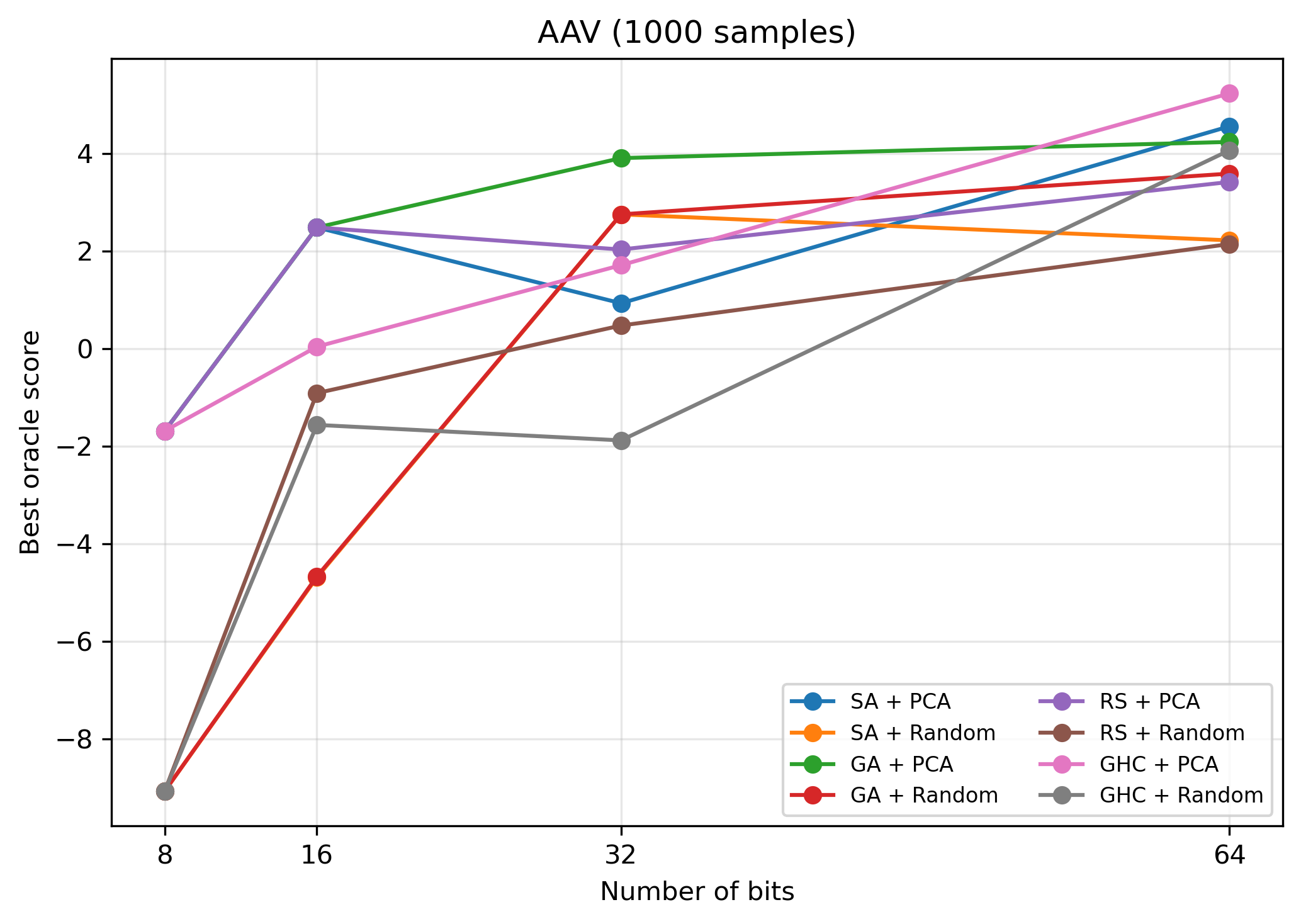}
\caption{AAV with 1000 training samples.}
\end{subfigure}
\hfill
\begin{subfigure}[t]{0.48\textwidth}
\centering
\includegraphics[width=\linewidth]{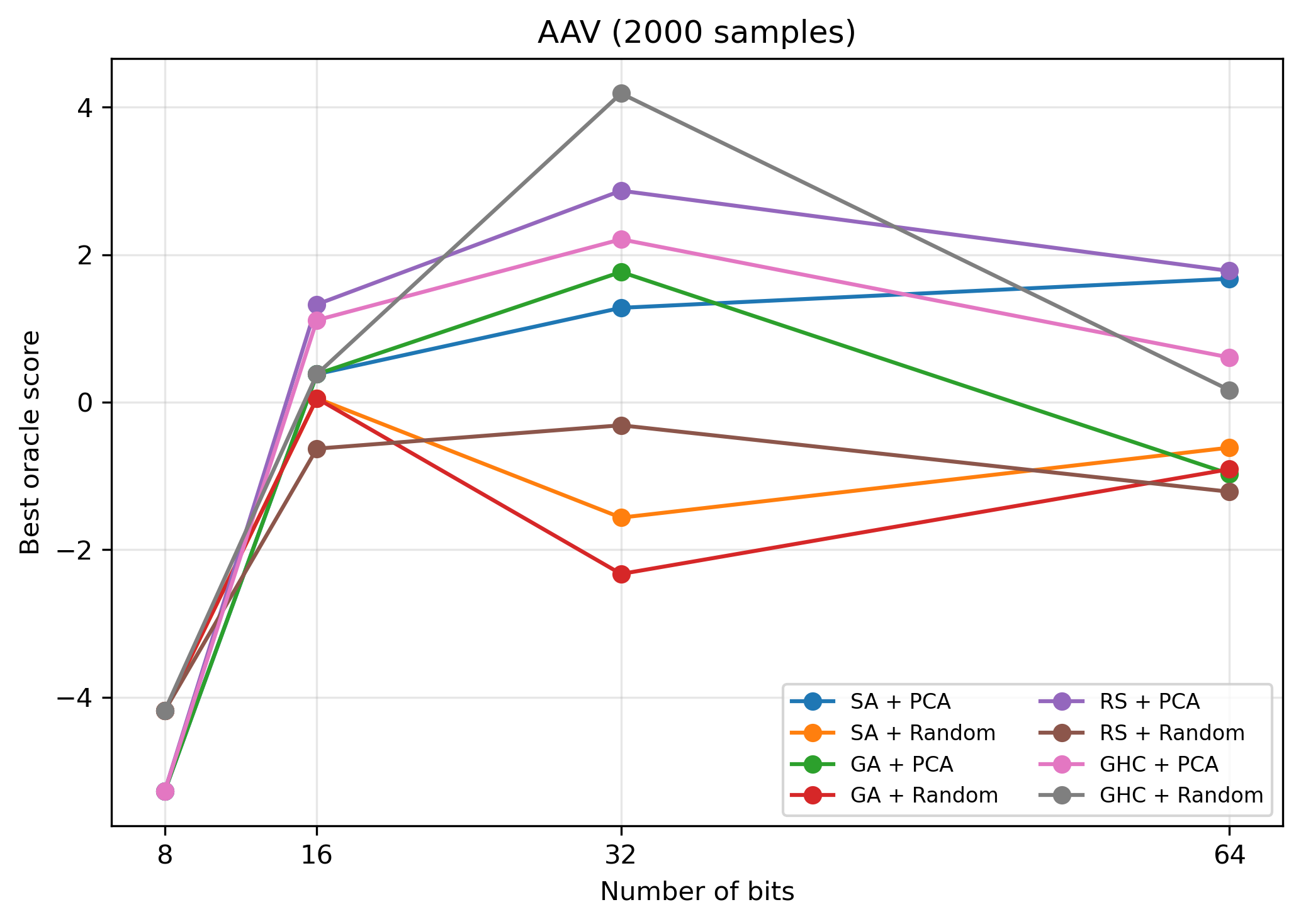}
\caption{AAV with 2000 training samples.}
\end{subfigure}
\vspace{0.5em}
\begin{subfigure}[t]{0.48\textwidth}
\centering
\includegraphics[width=\linewidth]{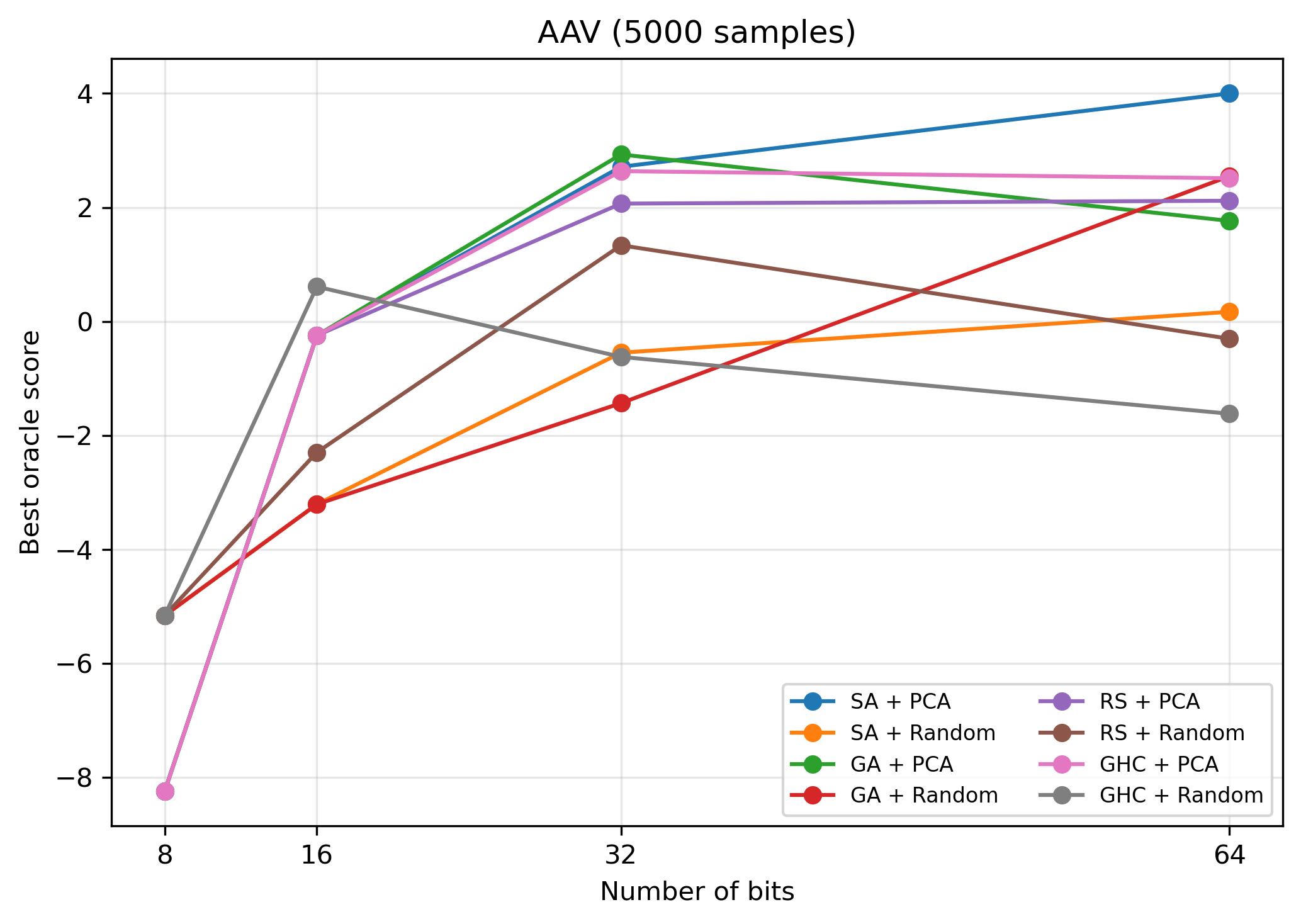}
\caption{AAV with 5000 training samples.}
\end{subfigure}
\hfill
\begin{subfigure}[t]{0.48\textwidth}
\centering
\includegraphics[width=\linewidth]{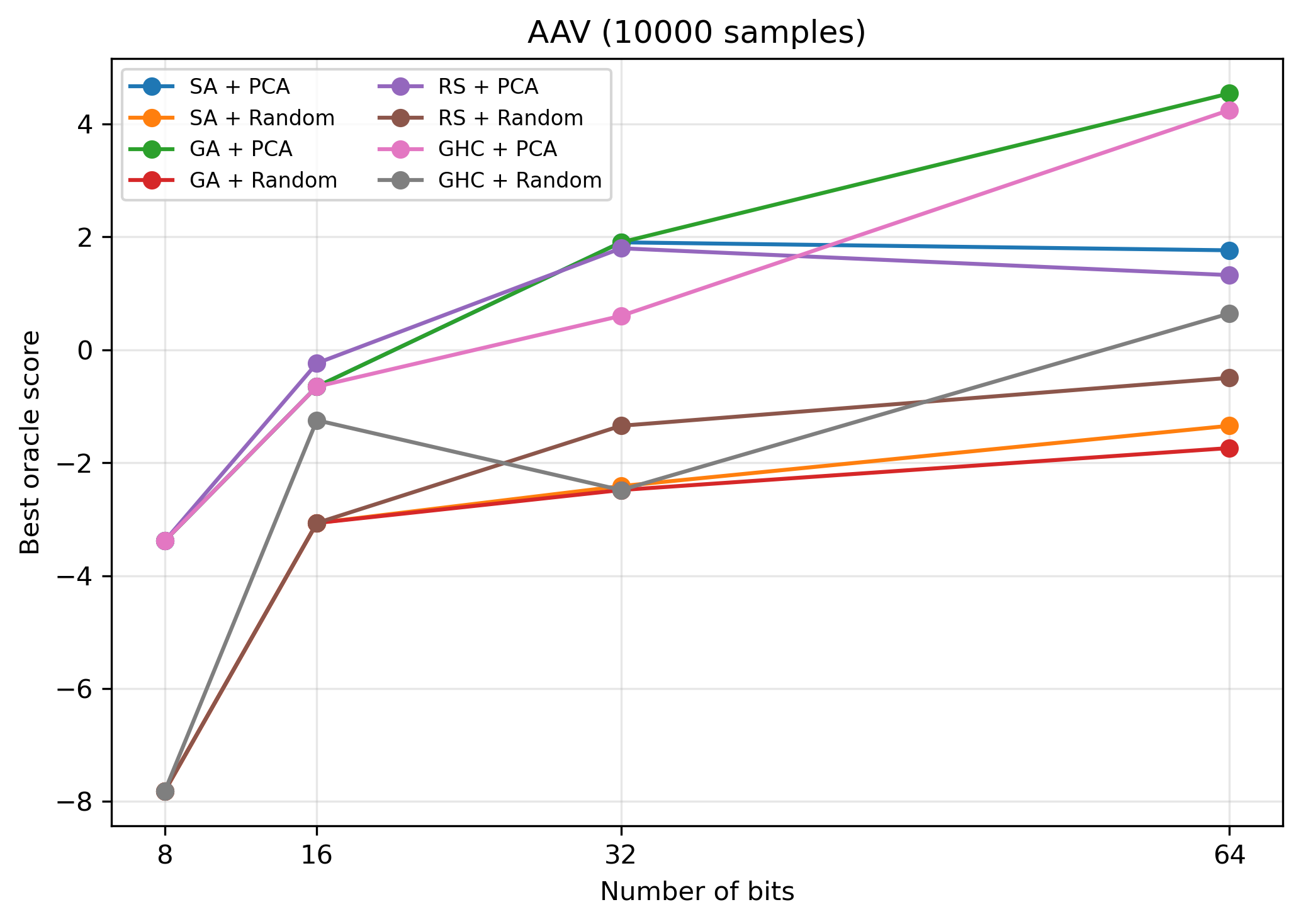}
\caption{AAV with 10000 training samples.}
\end{subfigure}
\vspace{0.5em}
\caption{End-to-end design performance on AAV across different training sizes. The x-axis shows the number of bits and the y-axis shows the best oracle score. Each line corresponds to an optimizer-representation pair, including simulated annealing (SA), genetic algorithm (GA), random search (RS), and greedy hill climbing (GHC), combined with either PCA-based or random-projection binary latent representations.}
\label{fig:appendix_aav_methods}
\end{figure*}

\begin{table*}[t]
\centering
\caption{Detailed end-to-end design results on GFP with 1000 training samples.}
\label{tab:appendix_gfp_1000}
\begin{tabular}{l l l c c}
\toprule
Bits & Representation & Best optimizer & Best score $\uparrow$ & Top-10 mean $\uparrow$ \\
\midrule
\multirow{2}{*}{8} & PCA & LBO & 2.616 & 2.524 \\
 & Random & LBO & 2.612 & 2.531 \\
\midrule
\multirow{2}{*}{16} & PCA & RS & 3.950 & 3.247 \\
 & Random & RS & 3.598 & 2.488 \\
\midrule
\multirow{2}{*}{32} & PCA & GHC & 4.032 & 3.017 \\
 & Random & GA & 3.690 & 2.525 \\
\midrule
\multirow{2}{*}{64} & PCA & RS & 4.119 & 3.903 \\
 & Random & GA & 4.170 & 3.312 \\
\bottomrule
\end{tabular}
\end{table*}

\begin{table*}[t]
\centering
\caption{Detailed end-to-end design results on GFP with 2000 training samples.}
\label{tab:appendix_gfp_2000}
\begin{tabular}{l l l c c}
\toprule
Bits & Representation & Best optimizer & Best score $\uparrow$ & Top-10 mean $\uparrow$ \\
\midrule
\multirow{2}{*}{8} & PCA & GA & 2.669 & 2.669 \\
 & Random & LBO & 2.669 & 2.267 \\
\midrule
\multirow{2}{*}{16} & PCA & GA & 3.503 & 3.212 \\
 & Random & LBO & 3.002 & 2.358 \\
\midrule
\multirow{2}{*}{32} & PCA & GA & 3.909 & 3.449 \\
 & Random & GHC & 3.448 & 2.693 \\
\midrule
\multirow{2}{*}{64} & PCA & GA & 4.027 & 3.549 \\
 & Random & LBO & 3.316 & 1.900 \\
\bottomrule
\end{tabular}
\end{table*}

\begin{table*}[t]
\centering
\caption{Detailed end-to-end design results on GFP with 5000 training samples.}
\label{tab:appendix_gfp_5000}
\begin{tabular}{l l l c c}
\toprule
Bits & Representation & Best optimizer & Best score $\uparrow$ & Top-10 mean $\uparrow$ \\
\midrule
\multirow{2}{*}{8} & PCA & GA & 2.653 & 2.653 \\
 & Random & LBO & 2.653 & 2.633 \\
\midrule
\multirow{2}{*}{16} & PCA & GHC & 3.656 & 3.117 \\
 & Random & LBO & 3.568 & 2.356 \\
\midrule
\multirow{2}{*}{32} & PCA & LBO & 4.157 & 2.918 \\
 & Random & GHC & 3.815 & 2.228 \\
\midrule
\multirow{2}{*}{64} & PCA & GHC & 4.270 & 3.972 \\
 & Random & LBO & 3.330 & 2.619 \\
\bottomrule
\end{tabular}
\end{table*}

\begin{table*}[t]
\centering
\caption{Detailed end-to-end design results on GFP with 10000 training samples.}
\label{tab:appendix_gfp_10000}
\begin{tabular}{l l l c c}
\toprule
Bits & Representation & Best optimizer & Best score $\uparrow$ & Top-10 mean $\uparrow$ \\
\midrule
\multirow{2}{*}{8} & PCA & LBO & 3.162 & -3.738 \\
 & Random & LBO & 2.785 & -4.556 \\
\midrule
\multirow{2}{*}{16} & PCA & LBO & 4.712 & 3.452 \\
 & Random & GHC & 4.502 & 3.029 \\
\midrule
\multirow{2}{*}{32} & PCA & GA & 5.287 & 4.823 \\
 & Random & GHC & 4.806 & 3.829 \\
\midrule
\multirow{2}{*}{64} & PCA & GHC & 5.554 & 5.042 \\
 & Random & GA & 4.758 & 3.003 \\
\bottomrule
\end{tabular}
\end{table*}

\begin{table*}[t]
\centering
\caption{Detailed end-to-end design results on AAV with 1000 training samples.}
\label{tab:appendix_aav_1000}
\begin{tabular}{l l l c c}
\toprule
Bits & Representation & Best optimizer & Best score $\uparrow$ & Top-10 mean $\uparrow$ \\
\midrule
\multirow{2}{*}{8} & PCA & GA & -1.687 & -4.039 \\
 & Random & LBO & -6.492 & -7.534 \\
\midrule
\multirow{2}{*}{16} & PCA & SA & 2.484 & -1.568 \\
 & Random & RS & -0.913 & -6.523 \\
\midrule
\multirow{2}{*}{32} & PCA & GA & 3.906 & -0.680 \\
 & Random & SA & 2.753 & -3.336 \\
\midrule
\multirow{2}{*}{64} & PCA & GHC & 5.232 & 2.962 \\
 & Random & GHC & 4.066 & -1.557 \\
\bottomrule
\end{tabular}
\end{table*}

\begin{table*}[t]
\centering
\caption{Detailed end-to-end design results on AAV with 2000 training samples.}
\label{tab:appendix_aav_2000}
\begin{tabular}{l l l c c}
\toprule
Bits & Representation & Best optimizer & Best score $\uparrow$ & Top-10 mean $\uparrow$ \\
\midrule
\multirow{2}{*}{8} & PCA & LBO & -4.213 & -6.054 \\
 & Random & LBO & -3.352 & -5.891 \\
\midrule
\multirow{2}{*}{16} & PCA & RS & 1.325 & -0.440 \\
 & Random & LBO & 1.769 & -3.162 \\
\midrule
\multirow{2}{*}{32} & PCA & RS & 2.866 & 0.349 \\
 & Random & GHC & 4.184 & -2.674 \\
\midrule
\multirow{2}{*}{64} & PCA & RS & 1.779 & -0.353 \\
 & Random & GHC & 0.164 & -1.511 \\
\bottomrule
\end{tabular}
\end{table*}

\begin{table*}[t]
\centering
\caption{Detailed end-to-end design results on AAV with 5000 training samples.}
\label{tab:appendix_aav_5000}
\begin{tabular}{l l l c c}
\toprule
Bits & Representation & Best optimizer & Best score $\uparrow$ & Top-10 mean $\uparrow$ \\
\midrule
\multirow{2}{*}{8} & PCA & LBO & -1.727 & -6.113 \\
 & Random & LBO & -2.616 & -5.168 \\
\midrule
\multirow{2}{*}{16} & PCA & RS & -0.245 & -2.136 \\
 & Random & GHC & 0.610 & -2.630 \\
\midrule
\multirow{2}{*}{32} & PCA & GA & 2.929 & 0.666 \\
 & Random & RS & 1.336 & -5.012 \\
\midrule
\multirow{2}{*}{64} & PCA & SA & 3.999 & 2.249 \\
 & Random & GA & 2.548 & -1.150 \\
\bottomrule
\end{tabular}
\end{table*}

\begin{table*}[t]
\centering
\caption{Detailed end-to-end design results on AAV with 10000 training samples.}
\label{tab:appendix_aav_10000}
\begin{tabular}{l l l c c}
\toprule
Bits & Representation & Best optimizer & Best score $\uparrow$ & Top-10 mean $\uparrow$ \\
\midrule
\multirow{2}{*}{8} & PCA & LBO & -1.903 & -3.723 \\
 & Random & LBO & -5.172 & -7.184 \\
\midrule
\multirow{2}{*}{16} & PCA & RS & -0.237 & -2.834 \\
 & Random & LBO & 1.197 & -4.205 \\
\midrule
\multirow{2}{*}{32} & PCA & SA & 1.904 & 0.727 \\
 & Random & LBO & -0.154 & -5.185 \\
\midrule
\multirow{2}{*}{64} & PCA & GA & 4.541 & 2.091 \\
 & Random & GHC & 0.644 & -2.174 \\
\bottomrule
\end{tabular}
\end{table*}

\clearpage
\bibliographystyle{unsrtnat}
\bibliography{main}

\end{document}